\documentclass{article}

     \usepackage[final]{neurips_2022}

\usepackage[utf8]{inputenc} % allow utf-8 input
\usepackage[T1]{fontenc}    % use 8-bit T1 fonts
\usepackage{comment}
\usepackage{url}            % simple URL typesetting
\usepackage{booktabs}       % professional-quality tables
\usepackage{amsfonts}       % blackboard math symbols
\usepackage{nicefrac}       % compact symbols for 1/2, etc.
\usepackage{microtype}      % microtypography
\usepackage[svgnames]{xcolor}
\usepackage{smile}
\usepackage{enumerate,natbib,mathrsfs}
\usepackage{algorithm}
\usepackage{algorithmic}
\usepackage{tablefootnote}
\usepackage{extarrows}
\usepackage[OT1]{fontenc}
\usepackage[bookmarks=false]{hyperref}
\hypersetup{
  pdftex,
  pdffitwindow=true,
  pdfstartview={FitH},
  pdfnewwindow=true,
  colorlinks,
  linktocpage=true,
  linkcolor=Red,
  urlcolor=Red,
  citecolor=Blue
}
\usepackage{stmaryrd}
\usepackage{graphicx}

\newcommand{\KL}{\mathrm{KL}}

\newcommand{\E}{\mathbb E}

\newcommand{\ids}{\text{IDS}}
\newcommand{\sids}{\text{s-IDS}}
\newcommand{\ts}{\text{TS}}

\newcommand{\Reg}{\mathfrak{R}}
\newcommand{\BR}{\mathfrak{BR}}
\renewcommand{\d}[1]{\operatorname{d}\!#1}

\def\botao{\color{red}}

\title{Regret Bounds for Information-Directed Reinforcement Learning}

\author{%
  Botao Hao\\
 Deepmind\\
  \texttt{haobotao000@gmail.com} \\
   \And
  Tor Lattimore \\
   Deepmind \\
    \texttt{lattimore@google.com} \\
}

\begin{document}

\maketitle

\begin{abstract}
Information-directed sampling (IDS) has revealed its potential as a data-efficient algorithm \citep{lu2021reinforcement} for reinforcement learning (RL). However, theoretical understanding of IDS for Markov Decision Processes (MDPs) is still limited. We develop novel information-theoretic tools to bound the information ratio and cumulative information gain about the learning target. Our theoretical results shed light on the importance of choosing the learning target such that the practitioners can balance the computation and regret bounds. As a consequence, we derive prior-free Bayesian regret bounds for \texttt{vanilla-IDS} which learns the whole environment under tabular finite-horizon MDPs. In addition, we propose a computationally-efficient \texttt{regularized-IDS} that maximizes an additive form rather than the ratio form and show that it enjoys the same regret bound as \texttt{vanilla-IDS}. With the aid of rate-distortion theory, we improve the regret bound by learning a surrogate, less informative environment. Furthermore, we extend our analysis to linear MDPs and prove similar regret bounds for Thompson sampling as a by-product.
\end{abstract}

\section{Introduction}
Information-directed sampling (IDS) is a \emph{design principle} proposed by \citep{russo2014learning, russo2018learning} that  optimizes the trade-off between \emph{information} and \emph{regret}. Comparing with other design principles such as UCB and Thompson sampling (TS), IDS can automatically adapt to different information-regret structures. As a result, IDS demonstrates impressive empirical performance \citep{russo2018learning} and outperforms UCB and TS in terms of asymptotic optimality \citep{kirschner2021asymptotically} and minimax optimality in heteroscedastic bandits
\citep{kirschner2018information} and
sparse linear bandits \citep{hao2021information}. 

In the context of full RL, mutiple works have examined the empirical performance of IDS \citep{nikolov2018information, lu2021reinforcement}. However, formal regret guarantee for IDS is still lacking. IDS minimizes a notion of \emph{information ratio} that is the ratio of per-episode regret and information gain about the learning target. While different choices of the learning target could lead to different regret bounds and computational methods, the most natural choice is the whole environment and we name the corresponding IDS as \texttt{vanilla-IDS}. 

In this work, we prove the first prior-free $\tilde O(\sqrt{S^3A^2H^4L})$ Bayesian regret bound for \texttt{vanilla-IDS}, where $S$ is the size of state space, $A$ is the size of action space, $H$ is the length of episodes and $L$ is the number of episodes. Computationally, \texttt{vanilla-IDS} needs to optimize over the full policy space, which is not efficient in general. To facilitate the computation, we consider its regularized form, named \texttt{regularized-IDS}, that can be solved by any dynamic programming solver. By carefully choosing the tunable parameter, we prove that \texttt{regularized-IDS} enjoys the same regret bound as \texttt{vanilla-IDS}.

Although learning the whole environment offers certain computational advantages, the agent could take too much information to learn the whole environment exactly. A key observation is that different states may correspond to the same value function which eventually determines the behavior of the optimal policy. Through the rate-distortion theory, we construct a surrogate environment that is less informative to learn but enough to identify the optimal policy. As a result, we propose \texttt{surrogate-IDS} that takes the surrogate environment as the learning target and prove a sharper $\tilde O(\sqrt{S^2A^2H^4L})$ bound for tabular MDPs. 

In the end, we extend our analysis to linear MDPs where we \emph{must} learn a surrogate environment due to potentially infinitely many states and derive a $\tilde O(dH^2\sqrt{T})$ Bayesian regret bound that matches the existing minimax lower bound up to a factor of $H$. As a by-product of our analysis, we also prove prior-free Bayesian regret bounds for TS under tabular and linear MDPs.

\paragraph{Related work}
In general, there are two ways to prove Bayesian regret bounds. The first is to introduce confidence sets such that the Bayesian regret bounds of TS can match the best possible frequentist regret bounds by UCB \citep{russo2014learning} and has been extended to RL by \cite{osband2013more, osband2014model, osband2019deep}. However, when the best possible bound for UCB is sub-optimal, this technique yields a sub-optimal Bayesian regret bound. In addition, this technique can only be used to analyze TS but not IDS.

The second is to decompose the Bayesian regret into a \emph{information ratio} term and a \emph{cumulative information gain} term and bound them by tools from information theory \citep{russo2016information}. This technique can be used to analyze both TS \citep{dong2018information, bubeck2020first} and IDS in bandits setting \citep{russo2014learning, liu2018information, kirschner2020asymptotically, hao2021information, hao2022contextual}, partial monitoring \citep{lattimore2019information, kirschner2020information, lattimore2021mirror} but not in RL as far as we know. One exception is \cite{lu2019information, lu2020information} who bounded the information ratio for a specific Dirichlet prior with additional assumptions. 

Frequentist regret bounds in episodic RL have received considerable attention recently. For tabular MDPs, several representative works include UCBVI \citep{azar2017minimax}, optimistic Q-learning \citep{jin2018q}, RLSVI \citep{russo2019worst}, UCB-Advantage \citep{zhang2020almost}, UCB-MQ \citep{menard2021ucb}. While our regret bounds are not state of the art, the
primary goal of this paper is to broaden the set of efficient RL design principles known to satisfy $\sqrt{T}$ regret bounds. 

For linear or linear mixture MDPs, several representative works include LSVI-UCB \citep{jin2020provably}, OPPO \citep{cai2020provably}, UCRL-VTR \citep{ayoub2020model, zhou2021nearly}, RLSVI \citep{zanette2020frequentist}. Notably, \cite{dann2021provably} derived minimax regret bounds for a variant of TS. Beyond linear cases, several works consider general function approximation based on Bellman rank \citep{jiang2017contextual}, eluder dimension \citep{wang2020reinforcement}, Bellman-eluder dimension \citep{jin2021bellman} and bilinear class \citep{du2021bilinear}.

It is worth mentioning the recent impressive work by \cite{foster2021statistical} who proposed a general Estimation-to-Decisions (E2D) design principle.  Although motivated by different design principles, E2D shares the similar form as \texttt{regularized-IDS}. On one hand, \cite{foster2021statistical} mainly focuses on statistical complexity, while we offer a specific computationally-efficient algorithm thanks to the chain rule of mutual information and independent priors. On the other hand, while E2D tends to learn the whole environment, our theory in Section \ref{sec:surrogate} suggests learning a surrogate environment could yield better regret bounds.
\section{Preliminary}\label{sec:setup}

\paragraph{Finite-horizon MDPs} The environment is characterized by a finite-horizon time-inhomogeneous MDP, which is a tuple $\cE= (\cS, \cA, H,  \{P_h\}_{h=1}^H, \{r_h\}_{h=1}^H)$, where $\cS$ is the countable state space with $|\cS| = S$, $\cA$ is the finite action space with $|\cA|=A$, $H$ is the episode length, %$\rho$ is the initial state distribution,
$P_h : \cS \times \cA \to \Delta_\cS$ is the transition probability kernel and $r_h : \cS \times \cA \to [0,1]$ is the reward function. For a finite set $\cS$, let $\Delta_{\cS}$ be
the set of probability distributions over $\cS$. We assume $\cS$, $\cA$, $r_h$ are known and deterministic while the transition probability kernel is unknown and random. Throughout the paper, we may write $P_h$ and $r_h$ explicitly depend on $\cE$ when necessary.

Let $\Theta_h=[0, 1]^{S\times A\times S}$ be the parameter space of $P_h$ and $\Theta=\Theta_1\times \cdots\times \Theta_H$ be the full parameter space. We assume $\rho_h$ is the prior probability measure for $P_h$ on $\Theta_h$ with Borel $\sigma$-algebra and $\rho=\rho_1\otimes\cdots \otimes \rho_H$ as the product prior probability measure for the whole environment on $\Theta$ with Borel $\sigma$-algebra. This ensures the priors over different layers are independent and the prior is assumed to be known to the learner.

 \paragraph{Interaction protocol} An agent interacts with a finite-horizon MDP as follows. The initial state $s_1^\ell$ is assumed to be fixed over episodes. In each episode $\ell\in[L]$ and each layer $h \in [H]$, the 
agent observes a state $s_h^\ell$, takes an action $a_h^\ell$, and receives a reward $r_h^\ell$. Then, the environment evolves to a random next state $s_{h+1}^\ell$ according 
to distribution $P_h(\cdot | s_h^\ell,a_h^\ell)$. The episode terminates when $s_{H+1}$ is reached and is reset to the initial state.

Denote $\cH_{\ell, h}$ as the history of episode $\ell$ up to layer $h$, e.g.,
$\cH_{\ell, h} = (s_1^\ell, a_1^\ell, r_1^\ell, \ldots, s_h^\ell, a_h^\ell, r_h^\ell )$ and the set of such possible history is
$
    \Omega_{h} =\prod_{i=1}^{h}(\cS\times \cA\times [0,1])\,.
$
Let $\cD_\ell = (\cH_{1, H}, \ldots, \cH_{\ell-1, H})$ as the entire history up to episode $\ell$ with $\cD_1=\emptyset$. %and the set of possible history up to episode $\ell$ as $\Pi_{t=1}^{\ell-1}\Omega_H$.
A policy $\pi$ is a collection of (possibly randomised) mappings  $(\pi_1, \ldots, \pi_{H})$ where each $\pi_h$ maps an element from $\Omega_{h-1}\times \cS$ to $\Delta(\cA)$ and $\Pi$ is the whole policy class.
A \emph{stationary policy} chooses actions based on only the current state and current layer. The set of such policies is denoted by $\Pi_{\text{S}}$ where we denote $\pi_h(a|s)$ as the probability that the agent chooses action $a$ at state $s$ and layer $h$.  %If for all $s$ and $h$, $\pi_h(a|s)=1$ for some $a$, we call $\pi$ as a \emph{deterministic policy}. If there exists some $s$ and $h$ such that for $a\neq a'$, $\pi_h(a|s)>0$ and $\pi_h(a'|s)>0$, we call $\pi$ as a stochastic policy.
%A non-stationary policy $\pi$ is a collection of $H$ functions $\pi_h: \cS\to \Delta_{\cA}$, that maps states to distributions over actions. 

\paragraph{Value function} For each $h\in [H]$ and a policy $\pi$, the value function $V_{h,\pi}^\cE: \cS \to \mathbb{R} $ is defined as the expected value of cumulative  rewards received under policy $\pi$ when starting from an arbitrary state at $h$th layer; that is,
\begin{equation*}
 V^{\cE}_{h, \pi}(s) := \E_{\pi}^{\cE}\left[\sum_{h' = h}^H r_{h'}(s_{h'}, a_{h'})  \bigg | s_h = s\right]\,,
\end{equation*}
where $\mathbb E_{\pi}^{\cE}$ denotes the expectation over the sample path generated under policy $\pi$ and environment $\cE$. We adapt the convention that $V_{H+1,\pi}^\cE(\cdot)=0$. There always exists an optimal policy $\pi^*$ which gives the optimal value $V^\cE_{h,\pi^*}(s) = \max_{\pi\in\Pi_{\text{S}}}V_{h,\pi}^\cE(s)$ for all $s\in\cS$ and $h\in[H]$. Note that in the Bayesian setting, $\pi^*$ is a function of $\cE$ so it is also a random variable. In addition, we define the action-value function as follows:
\begin{equation*}
    Q_{h,\pi}^{\cE}(s,a) := \E^{\cE}_{\pi}\left[\sum_{h' = h}^H r_{h'}(s_{h'}, a_{h'})  \bigg | s_h = s, a_h=a\right]\,,
\end{equation*}
which satisfies the Bellman equation:
$Q_{h,\pi}^{\cE}(s,a) = r_h(s,a) + \mathbb E_{s'\sim P_h(\cdot|s,a)}[V_{h+1,\pi}^{\cE}(s')].
$ 
Furthermore, we denote the \emph{state-action occupancy measure} as 
$$
    d_{h,\pi}^{\cE}(s,a) = \mathbb P^{\cE}_ \pi(s_h=s, a_h=a)\,,
$$
where we denote $\mathbb P^{\cE}_\pi$ as the law of the sample path generated under policy $\pi$ and environment $\cE$.

\paragraph{Bayesian regret} The agent interacts with the environment for $L$ episodes and the total number of steps is $T = LH$. The expected cumulative regret of an algorithm $\pi=\{\pi^{\ell}\}_{\ell=1}^L$ with respect to an environment $\cE$ is defined as
\begin{equation*}
    \Reg_L(\cE, \pi) = \mathbb E\left[\sum_{\ell=1}^L\left(V_{1,\pi^*}^\cE(s_1^\ell)-V_{1,\pi^\ell}^\cE(s_1^\ell)\right)\right]\,,
\end{equation*}
where the expectation is taken with respect to the randomness of $\pi^{\ell}$.
The Bayesian regret then is defined as 
\begin{equation*}
    \BR_L(\pi)=\mathbb E[ \Reg_L(\cE, \pi)]\,,
\end{equation*}
where the expectation is taken with respect to the prior distribution of $\cE$. %If a regret bound is agnostic to the choice of the prior, we call it a \emph{prior-free Bayesian regret bound}. 
At each episode, TS finds
\begin{equation*}
    \pi_{\ts}^\ell=\argmax_{\pi\in\Pi}V_{1,\pi}^{\cE_\ell}(s_1^\ell)\,,
\end{equation*}
where $\cE_\ell$ is a sample from the posterior distribution of $\cE$, e.g., $ \cE_\ell\sim\mathbb P(\cE\in\cdot|\cD_\ell)$ .

%and $\mathbb E[\mathbb I_{\ell}(X;Y)] = \mathbb I(X;Y|\cD_{\ell})$.

%\subsection{Linear MDP and linear mixture MDP}

%\begin{definition}[Linear mixture MDP \citep{ayoub2020model, zhou2021provably}]\label{def:linear_mixture_MDP}
%Let $\phi:\cS\times \cA\times \cS\to\mathbb R^d$ be a feature map that satisfies $\|\phi(\cdot|s,a)\|_2\leq 1$ for any $(s, a)\in\cS\times \cA$. A MDP is called a linear mixture MDP if for any $h\in[H]$, there exist vectors $\theta_h\in\mathbb R^d$ such that for any $(s,a)\in\cS\times \cA$, we have
%\begin{equation*}
%    P_h(\cdot|s,a) = \langle \phi(\cdot|s,a), \theta_h\rangle\,.
%\end{equation*}
%Let us denote $\Theta^{\text{LinM}}$ be the parameter space of linear mixture MDP and assume $
%   \|\theta_h\|_2\leq C_{\theta}\,.
%$
%\end{definition}
%As shown in \cite{zhou2021provably}, linear MDP and linear mixture MDP are two different classes of MDPs since they are based on different feature mapping and one cannot be covered by the other.

\paragraph{Notations} Let $(\Omega, \cF, \mathbb P)$ as a measurable space. A random variable $X$ is a measureable function $X:\Omega\to E$ from a set of possible outcomes $\Omega$ to a measurable space $E$. Now $\mathbb P(X\in\cdot)$ is a probability measure that maps from $\cF$ to $[0, 1]$. $\cD_\ell$ is another random variable from $\Omega$ to a measurable space $Y$. Then $\mathbb P(X\in\cdot|\cD_\ell)$ is a probability kernel that maps from $\Omega \times \cF\to[0, 1]$.

We write $\mathbb P_{\ell}(\cdot) = \mathbb P(\cdot|\cD_{\ell})$, $\mathbb E_{\ell}[\cdot] = \mathbb E[\cdot|\cD_\ell]$ and also define the conditional mutual information 
$\mathbb I_\ell(X;Y) = D_{\KL}(\mathbb P((X, Y)\in \cdot|\cD_\ell)||\mathbb P(X\in\cdot|\cD_\ell)\otimes\mathbb P(Y\in\cdot|\cD_\ell)). 
$ For a random variable $\chi$ we define:
\begin{equation*}
    \mathbb I_{\ell}^{\pi}(\chi; \cH_{\ell, h}) = D_{\KL}(\mathbb P_{\ell,\pi}((\chi, \cH_{\ell, h})\in \cdot)||\mathbb P_{\ell,\pi}(\chi\in\cdot)\otimes\mathbb P_{\ell,\pi}(\cH_{\ell, h}\in\cdot))\,,
\end{equation*}
where $\mathbb P_{\ell,\pi}$ is the law of $\chi$ and the history induced by policy $\pi$ interacting with a sample from the posterior distribution of $\cE$ given $\cD_\ell$. We define $\bar{\cE}_\ell$ as the mean MDP where for each state-action pair $(s,a)$, $P_{h}^{\bar\cE_{\ell}}(\cdot|s,a)=\mathbb E_\ell[P_h^\cE(\cdot|s,a)]$ is the mean of posterior measure.

\section{Learning the whole environment}

%\textbf{Thompson sampling.} At the beginning of each episode,  IDS computes a deterministic policy $\pi^{\ell}_{\ts}=\arg\max_{\pi\in\Pi_{\text{S}}}V_{h,\pi}^{\cE_{\ell}}(s)$ for all $s\in\cS$ and $h\in[H]$ where $\cE_{\ell}$ is a sample from the posterior distribution of $\cE$.
The core design of IDS for RL relies on a notion of \emph{information ratio}.
The information ratio for a policy $\pi$ at episode $\ell$ is defined as
\begin{equation}\label{def:information_ratio}
     \Gamma_\ell(\pi, \chi):=\frac{(\mathbb E_\ell[V_{1,\pi^*}^{\cE}(s_1^\ell)-V_{1,\pi}^\cE(s_1^\ell)])^2}{\mathbb I_\ell^{\pi}(\chi; \cH_{\ell, H})}\,,
\end{equation}
where $\chi$ is the learning target to prioritize information sought by the agent.
%Note that $\Gamma_\ell$ is a function of $\cD_\ell$ and thus a random variable.
 The choice of $\chi$ plays a crucial role in designing the IDS and could lead to different regret bounds and computational methods. We first consider the most natural choice of $\chi$ which is the whole environment $\cE$. 
\subsection{Vanilla IDS} 
\texttt{Vanilla-IDS} takes the whole environment $\cE$ as the learning target and at the beginning of each episode,  the agent computes a stochastic policy:
\begin{equation}\label{eqn:vanilla_IDS}
  \pi^\ell_{\ids}=   \argmin_{\pi\in\Pi}\left[\Gamma_\ell(\pi):=\frac{(\mathbb E_\ell[V_{1,\pi^*}^{\cE}(s_1^\ell)-V_{1,\pi}^\cE(s_1^\ell)])^2}{\mathbb I_\ell^\pi(\cE; \cH_{\ell, H})}\right]\,.
\end{equation}
Define the worst-case information ratio $\Gamma^*$ such that $\Gamma_\ell(\pi_{\ids}^\ell)\leq \Gamma^*$ for any $\ell\in[L]$ almost surely.
The next theorem derives a generic regret bound for \texttt{vanilla-IDS} in terms of $\Gamma^*$ and the mutual information between $\cE$ and the history. 
\begin{theorem}\label{lemma:generic_bound}
A generic regret bound for \texttt{vanilla-IDS} is 
\begin{equation*}
    \BR_L(\pi_{\ids})\leq \sqrt{\mathbb E[\Gamma^*]\mathbb I\left(\cE;\cD_{L+1}\right)L}\,.
\end{equation*}
\end{theorem}
The proof is deferred to Appendix \ref{sec:proof_lemma:generic_bound} and follows standard information-theoretical regret decomposition and the chain rule of mutual information that originally was exploited by \cite{russo2014learning}. For tabular MDPs, it remains to bound the $\mathbb E[\Gamma^*]$ and $\mathbb I\left(\cE;\cD_{L+1}\right)$  separately.

%\begin{remark}
%So far we put a prior on the transition kernel and define the information gain with respect to the transition kernel. Another way is to define the information gain with respect to the optimal policy which could enjoy smaller cumulative information gain. In order to use minimax duality theorem, one could also put priors on the loss sequences such that the information gain is defined with respect to the optimal policy in the hindsight. However, we conjecture it is impossible to bound the information ratio for this way.
%\end{remark}

\begin{lemma}\label{lemma:information_ratio_bound}
The worst-case information ratio for tabular MDPs is upper bounded by
\begin{equation*}
   \mathbb E[\Gamma^*]\leq 2SAH^3\,.
\end{equation*}
\end{lemma}
We sketch the main steps of the proof and defer the full proof to Appendix \ref{sec:proof_lemma:information_ratio_bound}.
\begin{proof}[Proof sketch]
Since \texttt{vanilla-IDS} minimizes the information ratio over all the policies, we can bound the information ratio of \texttt{vanilla-IDS} by the information ratio of TS. 
\begin{itemize}
    \item \emph{Step one.} Our regret decomposition uses the value function based on $\bar \cE_\ell$ as a bridge: 
\begin{equation*}\label{eqn:regret_decomposition}
     \begin{split}
          &\mathbb E_\ell\left[V_{1,\pi^*}^\cE(s_1^\ell)-V_{1,\pi^\ell_{\ts}}^\cE(s_1^\ell)\right]= \underbrace{\mathbb E_\ell\left[V_{1,\pi^*}^\cE(s_1^\ell)-V_{1,\pi^\ell_{\ts}}^{\bar{\cE}_\ell}(s_1^\ell)\right]}_{I_1} +  \underbrace{\mathbb E_\ell\left[V_{1,\pi^\ell_{\ts}}^{\bar{\cE}_\ell}(s_1^\ell)-V_{1,\pi^\ell_{\ts}}^\cE(s_1^\ell)\right]}_{I_2}\,.
     \end{split}
 \end{equation*}
 Note that conditional on $\cD_\ell$, the law of $\pi^{\ell}_{\ts}$ is the same as the law of $\pi^*$ and both $\pi^*$ and $\pi^\ell_{\ts}$ are independent of $\bar \cE_\ell$. This implies $\mathbb E_\ell[V_{1,\pi^\ell_{\ts}}^{\bar{\cE}_\ell}(s_1^\ell)]=\mathbb E_\ell[V_{1,\pi^*}^{\bar{\cE}_\ell}(s_1^\ell)]$. %because $\pi^\ell$ is the function of a sample drawn from the posterior distribution of $\cE$ while $\pi^*$ is the function of $\cE$. 
 \item \emph{Step two.} Denote $ \Delta_h^{\cE}(s,a) =\mathbb E_{s'\sim P_h^{ \cE}(\cdot|s,a)}[V_{h+1,\pi^*}^{ \cE}(s')]-\mathbb E_{s'\sim P_h^{\bar \cE}(\cdot|s,a)}[V_{h+1,\pi^*}^{ \cE}(s')]$ as the value function difference.
Inspired by \cite{foster2021statistical}, with the use of \emph{state-action occupancy measure} and Lemma \ref{lemma:bellman_residual}, we can derive
\begin{equation*}
   I_1 =\sum_{h=1}^H\mathbb E_{\ell}\left[ \sum_{(s,a)}\frac{d_{h,\pi^*}^{\bar \cE_\ell}(s,a)}{(\mathbb E_{\ell}[d_{h,\pi^*}^{\bar \cE_\ell}(s,a)])^{1/2}}(\mathbb E_{\ell}[d_{h,\pi^*}^{\bar \cE_\ell}(s,a)])^{1/2}\Delta_h^{\cE}(s,a)\right]\,.
\end{equation*}
Applying the Cauchy–Schwarz inequality and Pinsker's inequality (see Eqs.~\eqref{eqn:bound_state_occupancy}-\eqref{eqn:I12} in the appendix for details), we can obtain
\begin{equation*}
  I_1\leq \sqrt{SAH^3}\left(\sum_{h=1}^H\mathbb E_{\ell}\left[\mathbb E_{\pi^\ell_{\ts}}^{\bar \cE_\ell}\left[\frac{1}{2}D_{\KL}\left(P_h^{\cE}(\cdot|s_h^\ell,a_h^\ell)||P_h^{\bar \cE_\ell}(\cdot|s_h^\ell,a_h^\ell)\right)\right]\right]\right)^{1/2}\,,
\end{equation*}
where we interchange $\pi^\ell_{\ts}$ and $\pi^*$ again and  $\mathbb E_{\pi^\ell_{\ts}}^{\bar \cE_\ell}$ is taken with respect to $s_h^\ell, a_h^\ell$ and $\mathbb E_\ell$ is taken with respect to $\pi^\ell_{\ts}$ and $\cE$.
\item \emph{Step three.} It remains to establish the following equivalence of above KL-divergence and the information gain (Lemma \ref{lemma:information_gain_eqn}):
\begin{equation*}
   \sum_{h=1}^H \mathbb E_{\ell}\left[\mathbb E_{\pi^\ell_{\ts}}^{\bar \cE_\ell}\left[D_{\KL}\left(P_h^{\cE}(\cdot|s_h,a_h)||P_h^{\bar \cE_\ell}(\cdot|s_h,a_h)\right)\right]\right]= \mathbb I_\ell^{\pi_{\ts}^\ell}\left(\cE; \cH_{\ell, H}\right)\,.
\end{equation*}
A crucial step is to use the linearity of the expectation and the independence of priors over different layers (from the product prior as we assumed in Section \ref{sec:setup}) to show
\begin{equation*}
\begin{split}
\mathbb P_{\ell, \pi_{\ts}^{\ell}}(s_{h-1}=s, a_{h-1}=a)=\mathbb P_{\pi_{\ts}^{\ell}}^{\bar\cE_\ell}(s_{h-1}=s, a_{h-1}=a)\,.
      \end{split}
\end{equation*}
%This also verifies the importance of using $\bar \cE_\ell$ to decompose the regret. 
\end{itemize}
Combining Steps 1-3, we can reach the conclusion and the bound for $I_2$ is similar.
\end{proof}

The next lemma directly bounds the mutual information for tabular MDPs.
\begin{lemma}\label{lemma:information_gain}
The mutual information can be bounded by
\begin{equation*}
     \mathbb I(\cE; \cD_{L+1})\leq 2S^2AH\log\left(SLH\right)\,.
\end{equation*}
\end{lemma}
The proof relies on the construction of Bayes mixture density and a covering set for KL-divergence and is deferred to Appendix \ref{sec:proof_lemma:information_gain}.
Combining Theorem \ref{lemma:generic_bound}, Lemmas \ref{lemma:information_ratio_bound} and \ref{lemma:information_gain} yields the following:

\begin{theorem}[Regret bound for tabular MDPs]\label{thm:vanilla_ids}
Suppose $\pi_{\ids}=\{\pi^\ell_{\ids}\}_{\ell=1}^L$ is the vanilla IDS policy. The following Bayesian regret bound holds for tabular MDPs
\begin{equation*}
    \BR_L(\pi_{\ids})\leq \sqrt{8S^3A^2H^4L\log(SLH)}\,.
\end{equation*}
\end{theorem}
Although this regret bound is sub-optimal, this is the first sub-linear prior-free Bayesian regret bound for \texttt{vanilla-IDS}. 
\begin{remark}
It is worth mentioning that \cite{lu2019information, lu2020information} also derived Bayesian regret bound using information-theoretical tools but only hold for a specific Dirichlet prior as well other distribution-specific assumptions. Their proof heavily exploits the property of Dirichlet distribution and can not easily be extended to prior-free regret bounds.
\end{remark}

In the context of finite-horizon MDPs,
\cite{lu2021reinforcement} considered a \texttt{conditional-IDS} such that at each time step, conditional on $s_h^\ell$, \texttt{conditional-IDS} takes the action according to  
\begin{equation*}
   \pi_h(\cdot|s_h^\ell)= \argmin_{\nu\in\Delta_\cA}\frac{\left(\mathbb E_\ell\left[V_{h, \pi^*}^\cE(s_h^\ell)-Q_{h,\pi^*}^\cE(s_h^\ell, A_h)\right]\right)^2}{\mathbb I_\ell\left(\chi;(A_h, Q_{h,\pi^*}^\cE(s_h^\ell, A_h))\right)}\,,
\end{equation*}
where $A_h$ is sampled from $\nu$. \texttt{Conditional-IDS} defined the information ratio \emph{per-step} rather than \emph{per-episode} such that it only needs to optimize over action space rather than the policy space. This offers great computational benefits but there is no regret guarantee for \texttt{conditional-IDS}. Recently, \cite{hao2022contextual} has demonstrated the theoretical limitation of \texttt{conditional-IDS} in contextual bandits.

\subsection{Regularized IDS}
Computing an IDS policy practically usually involves two steps: 1. \emph{approximating the information ratio}; 2. \emph{optimizing the information ratio}. In bandits where the optimal policy is only a function of action space, optimizing Eq.~\eqref{eqn:vanilla_IDS} is a convex optimization problem and has
an optimal solution with at most two non-zero components (\citet[Proposition 6]{russo2018learning}). However in MDPs where the optimal policy is a mapping from the state space to the action space, \texttt{vanilla-IDS} needs to traverse two non-zero components over the full policy space which suggests the computational time might grow exponentially in $S$ and $H$.

To overcome this obstacle, we propose \texttt{regularized-IDS} that can be efficiently computed by any dynamic programming solver and enjoy the same regret bound as \texttt{vanilla-IDS}. At each episode $\ell$, \texttt{regularized-IDS} finds the policy:
\begin{equation}\label{eqn:Lagrangian_form}
    \pi^\ell_{\text{r-IDS}} = \argmax_{\pi\in\Pi} \mathbb E_{\ell}[V_{1,\pi}^\cE(s_1^\ell)]+\lambda \mathbb I_\ell\left(\cE; \cH_{\ell, H}^{\pi}\right)\,,
\end{equation}
where $\lambda>0$ is a tunable parameter. 

To approximate the objective function in Eq.~\eqref{eqn:Lagrangian_form}, we assume the access to a \emph{posterior sampling  oracle}.
\begin{definition}[Posterior sampling oracle]
Given a prior over $\cE$ and history $\cD_\ell$, the \emph{posterior sampling oracle}, \textit{SAMP}, is a subroutine which returns a sample from the posterior distribution $\mathbb P_\ell(\cE)$.  Multiple calls to the procedure result in independent samples.
\end{definition}

\begin{remark}
\textit{SAMP} can be exactly obtained when the conjugate prior such as Dirichlet distribution is put on the transition kernel. When one uses neural nets to estimate the model, \textit{SAMP} can be approximated by epistemic neural networks \citep{osband2021epistemic}, a general framework to quantify uncertainty for neural nets. The effectiveness of different epistemic neural networks such as deep ensemble, dropout and stochastic gradient MCMC has been examined empirically by \cite{osband2021evaluating}.
\end{remark}
We compute $\pi^\ell_{\text{r-IDS}}$ in two steps:
\begin{itemize}
    \item Firstly, we prove an equivalent form of the objective function in Eq.~\eqref{eqn:Lagrangian_form} using the chain rule of mutual information. Define $r'_h(s,a)$ as an \emph{augmented} reward function:
  \begin{equation*}
      r'_h(s,a) = r_h(s,a)+\lambda\int D_{\KL}\left(P_h^\cE(\cdot|s,a)||P_{h}^{\bar \cE_\ell}(\cdot|s,a)\right)\d \mathbb P_\ell(\cE)\,.
  \end{equation*} 
\begin{proposition}\label{prop:augmented_reward}
The following equivalence holds
\begin{equation*}
     \mathbb E_{\ell}[V_{1,\pi}^\cE(s_1^\ell)]+\lambda \mathbb I_\ell^{\pi}\left(\cE; \cH_{\ell, H}\right)= \mathbb E_{\pi}^{\bar\cE_\ell}\left[\sum_{h=1}^H r_h'(s_h, a_h)\right]\,.
\end{equation*}
\end{proposition}
The proof is deferred to Appendix \ref{sec:proof_prop:augmented_reward}.
\item  Given \textit{SAMP}, the augmented reward $r_h'$ and the MDP $\bar \cE_\ell$ can be well approximated by Monte Carlo sampling. Therefore, at each episode $\ell$, finding $\pi^\ell_{\text{r-IDS}}$ is equivalent to find an optimal policy based on a computable and augmented MDP $\{P_h^{\bar \cE_\ell}, r'_h\}_{h=1}^H$. This can be solved efficiently by any dynamic programming solver such as value iteration or policy iteration.  

\end{itemize}

In the end, we show that $\pi^\ell_{\text{r-IDS}}$ enjoys the same regret bound as \texttt{vanilla-IDS} when the tunable parameter is carefully chosen.
\begin{theorem}\label{thm:a-ids}
By choosing $\lambda = \sqrt{L\mathbb E[\Gamma^*]/\mathbb I(\cE;\cD_{L+1})}$, we have 
\begin{equation*}
     \BR_L(\pi^{\text{r-IDS}}) \leq \sqrt{\frac{3}{2}L\mathbb E[\Gamma^*]\mathbb I(\cE;\cD_{L+1})}\,.
\end{equation*}
\end{theorem}
The proof is deferred to Appendix \ref{sec:proof_thm:a-ids}. Let $M_1, M_2$ be upper bounds of $\mathbb E[\Gamma^*]$ and $\mathbb I(\cE;\cD_{L+1})$ respectively. In practice, we could conservatively choose $\lambda=\sqrt{LM_1/M_2}$ such that $ \BR_L(\pi^{\text{r-IDS}}) \leq\sqrt{3/2M_1M_2L}$. From Lemmas \ref{lemma:information_ratio_bound} and \ref{lemma:information_gain} for tabular MDPs, we could choose $M_1=2SAH^3$ and $M_2=2S^2AH\log(SLH)$.
\begin{remark}
\citet[Section 9.3]{russo2018learning} also considered a tunable version of IDS (for bandits) but took a square form of $\mathbb E_{\ell}[V_{1,\pi}^\cE(s_1^\ell)]$. While this makes no difference in bandits setting, this prevented us to use dynamic programming solver in RL setting. We are also inspired by \citet[Section 9.3]{foster2021statistical} who studied the relationship between information ratio and
Decision-Estimation Coefficient.
\end{remark}

\begin{comment}
{\botao
\begin{remark}[Relationship with intrinsic reward exploration]

\end{remark}

\begin{remark}
One issue with regularized-IDS is that for the current choice of $\lambda$, the regret bound of regularized-IDS can never achieve logarithm regret for some easy problem instances. In contrast, UCB-like policy as well as vanilla ratio form IDS can achieve logarithm regret.
\end{remark}
}

\begin{remark}[Relationship with Bayesian optimal policy]
Let $\mathbb Q$ as the prior probability measure of $\cE$. The Bayesian optimal policy is 
    \begin{equation*}
        \argmax_{\pi} \int_{\cE} \mathbb E\left[\sum_{\ell=1}^LV_1^{\pi^\ell}(s_1^\ell)\right] \d \mathbb Q(\cE)\,.
    \end{equation*}
\end{remark}
\end{comment}
\section{Learning a surrogate environment}\label{sec:surrogate}
When the state space is large, the agent could take too much information to learn exactly the whole environment $\cE$ which is reflected through $\mathbb I(\cE; \cD_{L+1})$. A key observation is that  different states may correspond to the same value function who eventually determines the behavior of the optimal policy. Based on the rate-distortion theory developed in \cite{dong2018information}, we reduce this redundancy and construct a surrogate environment that needs less information to learn.

\subsection{A rate distortion approach}
 The rate-distortion theory \citep{cover1991elements} addresses the problem of determining the minimal number of bits per symbol that should be communicated over a channel, so that the source (input signal) can be approximately reconstructed at the receiver (output signal) without exceeding an expected distortion. It was recently introduced to bandits community to develop sharper bounds for linear bandits \citep{dong2018information} and time-sensitive bandits \citep{russo2022satisficing}. We take a similar approach to construct a surrogate environment.

\paragraph{Surrogate environment} Suppose there exists a partition $\{\Theta_k\}_{k=1}^{K}$ over $\Theta$ such that for any $\cE, \cE'\in\Theta_k$ and any $k\in[K]$, we have 
\begin{equation}\label{eqn:cover}
    V_{1,\pi^*_\cE}^{\cE}(s_1^\ell)-V_{1, \pi^*_\cE}^{\cE'}(s_1^\ell)\leq \varepsilon\,,
\end{equation}
where $\varepsilon>0$ is the distortion tolerance and we write the optimal policy explicitly depending on the environment. 
%\begin{remark}
%The key difference with \citep{dong2018information} is that we construct a partition over the value function space rather than the environment space. 
%\end{remark}
Let $\zeta$ be a discrete random variable taking values in $\{1,\ldots, K\}$ that indicates the region $\cE$ lies such that $\zeta=k$ if and only if $\cE\in\Theta_k$. Therefore, $\zeta$ can be viewed as a statistic of $\cE$ and less informative than $\cE$ if $K$ is small. 

The next lemma shows the existence of the surrogate environment based on the partition. 
\begin{lemma}\label{lemma:surrogate_learning}
For any partition $\{\Theta_k\}^K_{k=1}$ and any $\ell\in[L]$, we can construct a surrogate environment $\tilde \cE_\ell^*\in \Theta$ which is a random MDP such that the law of $\tilde \cE_\ell^*$ only depends on $\zeta$ and 
\begin{equation}\label{eqn:additional_regret}
    \mathbb E_{\ell}\left[V_{1, \pi^*_\cE}^\cE(s_1^\ell)-V_{1, \pi^\ell_{\ts}}^\cE(s_1^\ell)\right]-\mathbb E_{\ell}\left[V_{1,\pi^*_\cE}^{\tilde \cE_\ell^*}(s_1^\ell)-V_{1,\pi^\ell_{\ts}}^{\tilde \cE_\ell^*}(s_1^\ell)\right]\leq \varepsilon\,.
\end{equation}
 \end{lemma}
The concrete form of $\tilde \cE_\ell^*$ is deferred to Eq.~\eqref{eqn:surrogate_env} in the appendix. %{\botao This lemma ensures the additional regret when the agent acts in the surrogate environment can be bounded.}
%Since the law of $\tilde \cE_\ell^*$ only depends on $\zeta$ that can only take $K$ values, learning towards $\tilde \cE_\ell^*$ needs less information than learning toward $\cE$. 
\paragraph{Surrogate IDS} We refer the IDS based on the surrogate environment $\tilde \cE_\ell^*$ as \texttt{surrogate-IDS} that minimizes
\begin{equation}\label{def:information_ratio_ids+}
    \pi^\ell_{\sids}= \argmin_{\pi\in \Pi}\frac{(\mathbb E_\ell[V_{1,\pi^*}^{\cE}(s_1^\ell)-V_{1,\pi}^\cE(s_1^\ell)]-\varepsilon)^2}{\mathbb I_\ell^\pi(\tilde \cE_\ell^*; \cH_{\ell, H})}\,,
\end{equation}
for some parameters $\varepsilon>0$ the will be chosen later. Denote the surrogate information ratio of TS as
\begin{equation*}
    \tilde{\Gamma} =\max_{\ell\in[L]} \frac{\left( \mathbb E_{\ell}\left[V_{1,\pi^*}^{\tilde \cE_\ell^*}(s_1^\ell)-V_{1,\pi_{\ts}^\ell}^{\tilde \cE_\ell^*}(s_1^\ell)\right]\right)^2}{\mathbb I_{\ell}^{\pi^\ell_{\ts}}(\tilde \cE_\ell^*; \cH_{\ell, H})}\,.
\end{equation*}
We first derive a generic regret bound for surrogate IDS in the following theorem.
\begin{theorem}\label{lemma:IDS+}
A generic regret bound for surrogate IDS is
\begin{equation*}
     \begin{split}
           \BR_L(\pi_{\sids})
         \leq\sqrt{\mathbb E[\tilde \Gamma]\mathbb I(\zeta; \cD_{L+1})L}+L\varepsilon\,.
     \end{split}
\end{equation*}
\end{theorem}
We defer the proof to Appendix \ref{sec:proof_lemma:IDS+}. Given $\zeta$, we have $\tilde \cE_\ell^*$ and $\cH_{\ell, H}$ are independent under the law of $\mathbb P_{\ell, \pi^\ell_{\sids}}$.
By the data processing inequality, the proof uses the fact that  $$\mathbb I_\ell^{\pi^\ell_{\sids}}(\tilde \cE_\ell^*; \cH_{\ell, H}) \leq \mathbb I_\ell^{\pi^\ell_{\sids}}(\zeta; \cH_{\ell, H})\,.
$$

Comparing with regret bound of \texttt{vanilla-IDS} in Lemma \ref{lemma:generic_bound}, the regret bound of \texttt{surrogate-IDS} depends on the information gain about $\zeta$ rather than the whole environment $\cE$. If there exists a partition with small covering number $K$, the agent could pay less information to learn. The second term $L\varepsilon$ is the price of distortions.

In the following, we will bound the $\mathbb E[\tilde \Gamma]$ and $\mathbb I\left(\zeta;\cD_{L+1}\right)$ for tabular and linear MDPs separately.

\subsection{Tabular MDPs}
We first show the existence of the partition required in Lemma \ref{lemma:surrogate_learning} for tabular MDPs and an upper bound of the covering number $K$.
\begin{lemma}\label{lemma:cover_value}
There exists a partition $\{\Theta_k^{\varepsilon}\}^K_{k=1}$ over $\Theta$ such that for any $k\in[K]$ and $\cE_1, \cE_2\in\Theta_k^{\varepsilon}$,
\begin{equation*}
    V_{1, \pi^*_{\cE_1}}^{\cE_1}(s_1)-V_{1,\pi^*_{\cE_1}}^{\cE_2}(s_1)\leq \varepsilon\,,
\end{equation*}
and the log covering number satisfies $\log(K)\leq SAH \log(4H^2/\varepsilon)$.
\end{lemma}
The proof is deferred to Lemma \ref{sec:proof_lemma:cover_value}.
For tabular MDPs, the mutual information between $\zeta$ and the history can be bounded by
\begin{equation*}
    \mathbb I(\zeta; \cD_{L+1})\leq \mathbb H(\zeta) \leq \log(K)\leq SAH\log(4H^2/\varepsilon)\,,
\end{equation*}
where $\mathbb H(\cdot)$ is the Shannon entropy.
Comparing with Lemma \ref{lemma:information_gain} when learning the whole environment,
learning the surrogate environment saves a factor of $S$ through the bound of mutual information. 
\begin{lemma}\label{lemma:ir_tabular}
The surrogate information ratio for tabular MDPs is upper bounded by
\begin{equation*}
    \mathbb E[\tilde{\Gamma}]\leq 2SAH^3\,.
\end{equation*}
\end{lemma}
The proof is the same as Lemma \ref{lemma:information_ratio_bound} and thus is omitted. Putting Lemmas \ref{lemma:cover_value}-\ref{lemma:ir_tabular} yields an improved bound for tabular MDPs using \texttt{surrogate-IDS}.

\begin{theorem}[Improved regret bound for tabular MDPs]
By choosing $\varepsilon=1/L$, the regret bound of \texttt{surrogate-IDS} for tabular MDPs satisfies
\begin{equation*}
    \BR_L(\pi_{\sids})\leq \sqrt{2S^2A^2H^4L\log(4HL)}\,.
\end{equation*}
\end{theorem}
For tabular MDPs, \texttt{surrogate-IDS} improves the regret bound of \texttt{vanilla-IDS} by a factor of $S$. However, it is still away from the minimax lower bound by a factor of $\sqrt{SAH}$. We conjecture \texttt{surrogate-IDS} can achieve the optimal bound with a price of lower order term but leave it as a future work.
\begin{remark}
Although the existence of $\tilde \cE_\ell^*$ is established using a constructive argument, finding $\tilde \cE_\ell^*$ needs a grid search and is not computationally efficient.
\end{remark}

\subsection{Linear MDPs}
We extend our analysis to
linear MDPs that is a fundamental model to study the theoretical properties of linear function approximations in RL. All the proofs are deferred to Appendix \ref{sec:proof_lemma:cover_linear_MDP}-\ref{sec:proof_lemma:information_ratio_linearMDP}.
\begin{definition}[Linear MDPs \citep{yang2019sample, jin2020provably}]\label{def:linear_MDP}
Let $\phi:\cS\times \cA\to\mathbb R^d$ be a feature map which
assigns to each state-action pair a $d$-dimensional feature vector and assume $\|\phi(s,a)\|_2\leq 1$.
An MDP is called a linear MDP if for any $h\in[H]$, there exist $d$ unknown (signed) measures $\psi_h^{1},\ldots, \psi_h^{d}$ over $\cS$, such that for any $(s,a)\in\cS\times \cA$, we have
\begin{equation*}
  P_h(\cdot|s,a) = \langle \phi(s,a), \psi_h(\cdot)\rangle\,,
\end{equation*}
where $\psi_h=(\psi_h^{1},\ldots, \psi_h^{d})$. Let us denote $\Theta^{\text{Lin}}$ be the parameter space of linear MDPs and assume $
   \left\|\sum_{s'}\psi_h(s')\right\|_2\leq C_{\psi}\,.
$
\end{definition}
Note that the degree of freedom of linear MDPs still depends on $S$ which implies that $\mathbb I(\cE; \cD_{L+1})$ may still scale with $S$. Therefore, we \emph{must} learn a surrogate environment rather than the whole environment for linear MDPs based on the current regret decomposition in Theorem \ref{thm:vanilla_ids}. We first show the existence of a partition over linear MDPs with the log covering number only depending on the feature dimension $d$. 
\begin{lemma}\label{lemma:cover_linear_MDP}
There exists a partition $\{\Theta_k^{\varepsilon}\}^K_{k=1}$ over $\Theta^{\text{Lin}}$ such that for any $k\in[K]$ and $\cE_1, \cE_2\in\Theta_k$,
\begin{equation*}
    V_{1, \pi^*_{\cE_1}}^{\cE_1}(s_1)-V_{1,\pi^*_{\cE_1}}^{\cE_2}(s_1)\leq \varepsilon\,,
\end{equation*}
and the log covering number satisfies $\log(K)\leq Hd\log(H^2C_{\psi}/\varepsilon+1)$.
\end{lemma}
 For linear MDPs, the mutual information can be bounded by
\begin{equation*}
    \mathbb I(\zeta; \cD_{L+1})\leq \mathbb H(\zeta) \leq \log(K)\leq Hd\log(H^2C_{\psi}/\varepsilon+1)\,.
\end{equation*}

\begin{lemma}\label{lemma:information_ratio_linearMDP}
The surrogate information ratio of linear MDPs is upper bounded by
$
    \mathbb E[\tilde{\Gamma}]\leq 4H^3d\,.
$
\end{lemma}

\begin{theorem}[Regret bound for linear MDPs]
By choosing $\varepsilon=1/L$, the regret bound of surrogate IDS for linear MDPs satisfies
\begin{equation*}
    \BR_L(\pi_{\sids})\leq \sqrt{4H^4d^2L\log(H^2C_{\psi}L+1)}+1\,.
\end{equation*}
\end{theorem}
This Bayesian bound improves the $O(d^{3/2}H^2\sqrt{L})$ frequentist
regret of LSVI-UCB \citep{jin2020provably} by a factor of $\sqrt{d}$ and matches the existing minimax lower bound $O(\sqrt{H^3d^2L})$ \citep{zhou2021nearly} up to a $H$ factor. However, we would like to emphasize that this is not an apples-to-apples comparison, since in general frequentist
regret bound is stronger than Bayesian regret bound.

\subsection{Regret bounds for TS}
As a direct application of our rate-distortion analysis, we provide Bayesian regret bounds for Thompson sampling. 
%Recall that at each episode, TS finds
%\begin{equation*}
%    \pi_{\ts}^\ell=\argmax_{\pi\in\Pi}V_{1,\pi}^{\cE_\ell}(s_1^\ell)\,,
%\end{equation*}
%where $\cE_\ell$ is a sample from the posterior distribution of $\cE$.

\begin{theorem}\label{thm:regret_TS}
A generic regret bound for TS is
\begin{equation*}
     \begin{split}
           \BR_L(\pi_{\ts})
         \leq\sqrt{\mathbb E[\tilde \Gamma]\mathbb I(\zeta; \cD_{L+1})L}+L\varepsilon\,.
     \end{split}
\end{equation*}
\end{theorem}
This implies for tabular and linear MDPs, TS has the same regret bound as \texttt{surrogate-IDS}. Note that the computation of TS does not need to involve the surrogate environment $\tilde \cE^*_\ell$ so once the posterior sampling oracle is available, computing the policy is efficient. Howevern when the worst-case information ratio cannot be optimally bounded by the information ratio of TS, IDS demonstrates better regret bounds than TS, such as bandits with graph feedback \citep{hao2022contextual} and sparse linear bandits \citep{hao2021information}. 
\section{Conclusion}
In this paper, we derive the first prior-free Bayesian regret bounds for information-directed RL under tabular and linear MDPs. Theoretically, it will be of great interest to see if any version of IDS can achieve the $O(\sqrt{SAH^3L})$ minimax lower bounds for tabular MDPs.

\section*{Acknowledgements}
We would like to thank Johannes Kirschner for helpful initial discussions.

\bibliographystyle{plainnat}%Used BibTeX style is unsrt
{\small
\bibliography{ref}

\begin{thebibliography}{44}
\providecommand{\natexlab}[1]{#1}
\providecommand{\url}[1]{\texttt{#1}}
\expandafter\ifx\csname urlstyle\endcsname\relax
  \providecommand{\doi}[1]{doi: #1}\else
  \providecommand{\doi}{doi: \begingroup \urlstyle{rm}\Url}\fi

\bibitem[Ayoub et~al.(2020)Ayoub, Jia, Szepesvari, Wang, and
  Yang]{ayoub2020model}
Alex Ayoub, Zeyu Jia, Csaba Szepesvari, Mengdi Wang, and Lin Yang.
\newblock Model-based reinforcement learning with value-targeted regression.
\newblock In \emph{International Conference on Machine Learning}, pages
  463--474. PMLR, 2020.

\bibitem[Azar et~al.(2017)Azar, Osband, and Munos]{azar2017minimax}
Mohammad~Gheshlaghi Azar, Ian Osband, and R{\'e}mi Munos.
\newblock Minimax regret bounds for reinforcement learning.
\newblock In \emph{International Conference on Machine Learning}, pages
  263--272. PMLR, 2017.

\bibitem[Bubeck and Sellke(2020)]{bubeck2020first}
S{\'e}bastien Bubeck and Mark Sellke.
\newblock First-order bayesian regret analysis of thompson sampling.
\newblock In \emph{Algorithmic Learning Theory}, pages 196--233. PMLR, 2020.

\bibitem[Cai et~al.(2020)Cai, Yang, Jin, and Wang]{cai2020provably}
Qi~Cai, Zhuoran Yang, Chi Jin, and Zhaoran Wang.
\newblock Provably efficient exploration in policy optimization.
\newblock In \emph{International Conference on Machine Learning}, pages
  1283--1294. PMLR, 2020.

\bibitem[Cover and Thomas(1991)]{cover1991elements}
T.M. Cover and J.A. Thomas.
\newblock \emph{Elements of Information Theory}.
\newblock Wiley Series in Telecommunications and Signal Processing. Wiley,
  1991.
\newblock ISBN 9780471062592.
\newblock URL \url{https://books.google.com/books?id=CX9QAAAAMAAJ}.

\bibitem[Dann et~al.(2021)Dann, Mohri, Zhang, and Zimmert]{dann2021provably}
Christoph Dann, Mehryar Mohri, Tong Zhang, and Julian Zimmert.
\newblock A provably efficient model-free posterior sampling method for
  episodic reinforcement learning.
\newblock \emph{Advances in Neural Information Processing Systems}, 34, 2021.

\bibitem[Dong and Roy(2018)]{dong2018information}
Shi Dong and Benjamin~Van Roy.
\newblock An information-theoretic analysis for thompson sampling with many
  actions.
\newblock In \emph{Proceedings of the 32nd International Conference on Neural
  Information Processing Systems}, pages 4161--4169, 2018.

\bibitem[Du et~al.(2021)Du, Kakade, Lee, Lovett, Mahajan, Sun, and
  Wang]{du2021bilinear}
Simon Du, Sham Kakade, Jason Lee, Shachar Lovett, Gaurav Mahajan, Wen Sun, and
  Ruosong Wang.
\newblock Bilinear classes: A structural framework for provable generalization
  in rl.
\newblock In \emph{International Conference on Machine Learning}, pages
  2826--2836. PMLR, 2021.

\bibitem[Foster et~al.(2021)Foster, Kakade, Qian, and
  Rakhlin]{foster2021statistical}
Dylan~J Foster, Sham~M Kakade, Jian Qian, and Alexander Rakhlin.
\newblock The statistical complexity of interactive decision making.
\newblock \emph{arXiv preprint arXiv:2112.13487}, 2021.

\bibitem[Hao et~al.(2021)Hao, Lattimore, and Deng]{hao2021information}
Botao Hao, Tor Lattimore, and Wei Deng.
\newblock Information directed sampling for sparse linear bandits.
\newblock \emph{Advances in Neural Information Processing Systems}, 34, 2021.

\bibitem[Hao et~al.(2022)Hao, Lattimore, and Qing]{hao2022contextual}
Botao Hao, Tor Lattimore, and Chao Qing.
\newblock Contextual information-directed sampling.
\newblock \emph{arXiv preprint}, 2022.

\bibitem[Jiang et~al.(2017)Jiang, Krishnamurthy, Agarwal, Langford, and
  Schapire]{jiang2017contextual}
Nan Jiang, Akshay Krishnamurthy, Alekh Agarwal, John Langford, and Robert~E
  Schapire.
\newblock Contextual decision processes with low bellman rank are
  pac-learnable.
\newblock In \emph{International Conference on Machine Learning}, pages
  1704--1713. PMLR, 2017.

\bibitem[Jin et~al.(2018)Jin, Allen-Zhu, Bubeck, and Jordan]{jin2018q}
Chi Jin, Zeyuan Allen-Zhu, Sebastien Bubeck, and Michael~I Jordan.
\newblock Is q-learning provably efficient?
\newblock \emph{Advances in neural information processing systems}, 31, 2018.

\bibitem[Jin et~al.(2020)Jin, Yang, Wang, and Jordan]{jin2020provably}
Chi Jin, Zhuoran Yang, Zhaoran Wang, and Michael~I Jordan.
\newblock Provably efficient reinforcement learning with linear function
  approximation.
\newblock In \emph{Conference on Learning Theory}, pages 2137--2143. PMLR,
  2020.

\bibitem[Jin et~al.(2021)Jin, Liu, and Miryoosefi]{jin2021bellman}
Chi Jin, Qinghua Liu, and Sobhan Miryoosefi.
\newblock Bellman eluder dimension: New rich classes of rl problems, and
  sample-efficient algorithms.
\newblock \emph{Advances in Neural Information Processing Systems}, 34, 2021.

\bibitem[Kirschner and Krause(2018)]{kirschner2018information}
Johannes Kirschner and Andreas Krause.
\newblock Information directed sampling and bandits with heteroscedastic noise.
\newblock In \emph{Conference On Learning Theory}, pages 358--384. PMLR, 2018.

\bibitem[Kirschner et~al.(2020{\natexlab{a}})Kirschner, Lattimore, and
  Krause]{kirschner2020information}
Johannes Kirschner, Tor Lattimore, and Andreas Krause.
\newblock Information directed sampling for linear partial monitoring.
\newblock In \emph{Conference on Learning Theory}, pages 2328--2369. PMLR,
  2020{\natexlab{a}}.

\bibitem[Kirschner et~al.(2020{\natexlab{b}})Kirschner, Lattimore, Vernade, and
  Szepesv{\'a}ri]{kirschner2020asymptotically}
Johannes Kirschner, Tor Lattimore, Claire Vernade, and Csaba Szepesv{\'a}ri.
\newblock Asymptotically optimal information-directed sampling.
\newblock \emph{arXiv preprint arXiv:2011.05944}, 2020{\natexlab{b}}.

\bibitem[Kirschner et~al.(2021)Kirschner, Lattimore, Vernade, and
  Szepesv{\'a}ri]{kirschner2021asymptotically}
Johannes Kirschner, Tor Lattimore, Claire Vernade, and Csaba Szepesv{\'a}ri.
\newblock Asymptotically optimal information-directed sampling.
\newblock In \emph{Conference on Learning Theory}, pages 2777--2821. PMLR,
  2021.

\bibitem[Lattimore and Gyorgy(2021)]{lattimore2021mirror}
Tor Lattimore and Andras Gyorgy.
\newblock Mirror descent and the information ratio.
\newblock In \emph{Conference on Learning Theory}, pages 2965--2992. PMLR,
  2021.

\bibitem[Lattimore and Szepesv{\'a}ri(2019)]{lattimore2019information}
Tor Lattimore and Csaba Szepesv{\'a}ri.
\newblock An information-theoretic approach to minimax regret in partial
  monitoring.
\newblock In \emph{Conference on Learning Theory}, pages 2111--2139. PMLR,
  2019.

\bibitem[Liu et~al.(2018)Liu, Buccapatnam, and Shroff]{liu2018information}
Fang Liu, Swapna Buccapatnam, and Ness Shroff.
\newblock Information directed sampling for stochastic bandits with graph
  feedback.
\newblock In \emph{Proceedings of the AAAI Conference on Artificial
  Intelligence}, volume~32, 2018.

\bibitem[Lu(2020)]{lu2020information}
Xiuyuan Lu.
\newblock \emph{Information-Directed Sampling for Reinforcement Learning}.
\newblock Stanford University, 2020.

\bibitem[Lu and Van~Roy(2019)]{lu2019information}
Xiuyuan Lu and Benjamin Van~Roy.
\newblock Information-theoretic confidence bounds for reinforcement learning.
\newblock \emph{arXiv preprint arXiv:1911.09724}, 2019.

\bibitem[Lu et~al.(2021)Lu, Van~Roy, Dwaracherla, Ibrahimi, Osband, and
  Wen]{lu2021reinforcement}
Xiuyuan Lu, Benjamin Van~Roy, Vikranth Dwaracherla, Morteza Ibrahimi, Ian
  Osband, and Zheng Wen.
\newblock Reinforcement learning, bit by bit.
\newblock \emph{arXiv preprint arXiv:2103.04047}, 2021.

\bibitem[M{\'e}nard et~al.(2021)M{\'e}nard, Domingues, Shang, and
  Valko]{menard2021ucb}
Pierre M{\'e}nard, Omar~Darwiche Domingues, Xuedong Shang, and Michal Valko.
\newblock Ucb momentum q-learning: Correcting the bias without forgetting.
\newblock In \emph{International Conference on Machine Learning}, pages
  7609--7618. PMLR, 2021.

\bibitem[Nikolov et~al.(2018)Nikolov, Kirschner, Berkenkamp, and
  Krause]{nikolov2018information}
Nikolay Nikolov, Johannes Kirschner, Felix Berkenkamp, and Andreas Krause.
\newblock Information-directed exploration for deep reinforcement learning.
\newblock In \emph{International Conference on Learning Representations}, 2018.

\bibitem[Osband and Van~Roy(2014)]{osband2014model}
Ian Osband and Benjamin Van~Roy.
\newblock Model-based reinforcement learning and the eluder dimension.
\newblock \emph{Advances in Neural Information Processing Systems}, 27, 2014.

\bibitem[Osband et~al.(2013)Osband, Russo, and Van~Roy]{osband2013more}
Ian Osband, Daniel Russo, and Benjamin Van~Roy.
\newblock (more) efficient reinforcement learning via posterior sampling.
\newblock \emph{Advances in Neural Information Processing Systems}, 26, 2013.

\bibitem[Osband et~al.(2019)Osband, Van~Roy, Russo, Wen,
  et~al.]{osband2019deep}
Ian Osband, Benjamin Van~Roy, Daniel~J Russo, Zheng Wen, et~al.
\newblock Deep exploration via randomized value functions.
\newblock \emph{J. Mach. Learn. Res.}, 20\penalty0 (124):\penalty0 1--62, 2019.

\bibitem[Osband et~al.(2021{\natexlab{a}})Osband, Wen, Asghari, Ibrahimi, Lu,
  and Van~Roy]{osband2021epistemic}
Ian Osband, Zheng Wen, Mohammad Asghari, Morteza Ibrahimi, Xiyuan Lu, and
  Benjamin Van~Roy.
\newblock Epistemic neural networks.
\newblock \emph{arXiv preprint arXiv:2107.08924}, 2021{\natexlab{a}}.

\bibitem[Osband et~al.(2021{\natexlab{b}})Osband, Wen, Asghari, Dwaracherla,
  Hao, Ibrahimi, Lawson, Lu, O'Donoghue, and Van~Roy]{osband2021evaluating}
Ian Osband, Zheng Wen, Seyed~Mohammad Asghari, Vikranth Dwaracherla, Botao Hao,
  Morteza Ibrahimi, Dieterich Lawson, Xiuyuan Lu, Brendan O'Donoghue, and
  Benjamin Van~Roy.
\newblock The neural testbed: Evaluating predictive distributions.
\newblock \emph{arXiv preprint arXiv:2110.04629}, 2021{\natexlab{b}}.

\bibitem[Russo(2019)]{russo2019worst}
Daniel Russo.
\newblock Worst-case regret bounds for exploration via randomized value
  functions.
\newblock \emph{Advances in Neural Information Processing Systems}, 32, 2019.

\bibitem[Russo and Van~Roy(2014)]{russo2014learning}
Daniel Russo and Benjamin Van~Roy.
\newblock Learning to optimize via information-directed sampling.
\newblock \emph{Advances in Neural Information Processing Systems}, 27, 2014.

\bibitem[Russo and Van~Roy(2016)]{russo2016information}
Daniel Russo and Benjamin Van~Roy.
\newblock An information-theoretic analysis of thompson sampling.
\newblock \emph{The Journal of Machine Learning Research}, 17\penalty0
  (1):\penalty0 2442--2471, 2016.

\bibitem[Russo and Van~Roy(2018)]{russo2018learning}
Daniel Russo and Benjamin Van~Roy.
\newblock Learning to optimize via information-directed sampling.
\newblock \emph{Operations Research}, 66\penalty0 (1):\penalty0 230--252, 2018.

\bibitem[Russo and Van~Roy(2022)]{russo2022satisficing}
Daniel Russo and Benjamin Van~Roy.
\newblock Satisficing in time-sensitive bandit learning.
\newblock \emph{Mathematics of Operations Research}, 2022.

\bibitem[Tang and Polyanskiy(2021)]{tang2021capacity}
Jennifer Tang and Yury Polyanskiy.
\newblock Capacity of noisy permutation channels.
\newblock \emph{arXiv preprint arXiv:2111.00559}, 2021.

\bibitem[Wang et~al.(2020)Wang, Salakhutdinov, and Yang]{wang2020reinforcement}
Ruosong Wang, Russ~R Salakhutdinov, and Lin Yang.
\newblock Reinforcement learning with general value function approximation:
  Provably efficient approach via bounded eluder dimension.
\newblock \emph{Advances in Neural Information Processing Systems},
  33:\penalty0 6123--6135, 2020.

\bibitem[Yang and Wang(2019)]{yang2019sample}
Lin Yang and Mengdi Wang.
\newblock Sample-optimal parametric q-learning using linearly additive
  features.
\newblock In \emph{International Conference on Machine Learning}, pages
  6995--7004. PMLR, 2019.

\bibitem[Zanette et~al.(2020)Zanette, Brandfonbrener, Brunskill, Pirotta, and
  Lazaric]{zanette2020frequentist}
Andrea Zanette, David Brandfonbrener, Emma Brunskill, Matteo Pirotta, and
  Alessandro Lazaric.
\newblock Frequentist regret bounds for randomized least-squares value
  iteration.
\newblock In \emph{International Conference on Artificial Intelligence and
  Statistics}, pages 1954--1964. PMLR, 2020.

\bibitem[Zhang(2021)]{zhang2021feel}
Tong Zhang.
\newblock Feel-good thompson sampling for contextual bandits and reinforcement
  learning.
\newblock \emph{arXiv preprint arXiv:2110.00871}, 2021.

\bibitem[Zhang et~al.(2020)Zhang, Zhou, and Ji]{zhang2020almost}
Zihan Zhang, Yuan Zhou, and Xiangyang Ji.
\newblock Almost optimal model-free reinforcement learningvia
  reference-advantage decomposition.
\newblock \emph{Advances in Neural Information Processing Systems},
  33:\penalty0 15198--15207, 2020.

\bibitem[Zhou et~al.(2021)Zhou, Gu, and Szepesvari]{zhou2021nearly}
Dongruo Zhou, Quanquan Gu, and Csaba Szepesvari.
\newblock Nearly minimax optimal reinforcement learning for linear mixture
  markov decision processes.
\newblock In \emph{Conference on Learning Theory}, pages 4532--4576. PMLR,
  2021.

\end{thebibliography}
}

%%%%%%%%%%%%%%%%%%%%%%%%%%%%%%%%%%%%%%%%%%%%%%%%%%%%%%%%%%%%
\section*{Checklist}

%%% BEGIN INSTRUCTIONS %%%
The checklist follows the references.  Please
read the checklist guidelines carefully for information on how to answer these
questions.  For each question, change the default \answerTODO{} to \answerYes{},
\answerNo{}, or \answerNA{}.  You are strongly encouraged to include a {\bf
justification to your answer}, either by referencing the appropriate section of
your paper or providing a brief inline description.  For example:
\begin{itemize}
  \item Did you include the license to the code and datasets? \answerNA{}
  \item Did you include the license to the code and datasets? \answerNA{}
  \item Did you include the license to the code and datasets? \answerNA{}
\end{itemize}
Please do not modify the questions and only use the provided macros for your
answers.  Note that the Checklist section does not count towards the page
limit.  In your paper, please delete this instructions block and only keep the
Checklist section heading above along with the questions/answers below.
%%% END INSTRUCTIONS %%%

\begin{enumerate}

\item For all authors...
\begin{enumerate}
  \item Do the main claims made in the abstract and introduction accurately reflect the paper's contributions and scope?
    \answerYes{}
  \item Did you describe the limitations of your work?
   \answerYes{}
  \item Did you discuss any potential negative societal impacts of your work?
   \answerNA{}
  \item Have you read the ethics review guidelines and ensured that your paper conforms to them?
    \answerYes{}
\end{enumerate}

\item If you are including theoretical results...
\begin{enumerate}
  \item Did you state the full set of assumptions of all theoretical results?
    \answerYes{}
        \item Did you include complete proofs of all theoretical results?
    \answerYes{}
\end{enumerate}

\item If you ran experiments...
\begin{enumerate}
  \item Did you include the code, data, and instructions needed to reproduce the main experimental results (either in the supplemental material or as a URL)?
    \answerNA{}
  \item Did you specify all the training details (e.g., data splits, hyperparameters, how they were chosen)?
     \answerNA{}
        \item Did you report error bars (e.g., with respect to the random seed after running experiments multiple times)?
     \answerNA{}
        \item Did you include the total amount of compute and the type of resources used (e.g., type of GPUs, internal cluster, or cloud provider)?
     \answerNA{}
\end{enumerate}

\item If you are using existing assets (e.g., code, data, models) or curating/releasing new assets...
\begin{enumerate}
  \item If your work uses existing assets, did you cite the creators?
    \answerNA{}
  \item Did you mention the license of the assets?
     \answerNA{}
  \item Did you include any new assets either in the supplemental material or as a URL?
     \answerNA{}
  \item Did you discuss whether and how consent was obtained from people whose data you're using/curating?
    \answerNA{}
  \item Did you discuss whether the data you are using/curating contains personally identifiable information or offensive content?
     \answerNA{}
\end{enumerate}

\item If you used crowdsourcing or conducted research with human subjects...
\begin{enumerate}
  \item Did you include the full text of instructions given to participants and screenshots, if applicable?
     \answerNA{}
  \item Did you describe any potential participant risks, with links to Institutional Review Board (IRB) approvals, if applicable?
    \answerNA{}
  \item Did you include the estimated hourly wage paid to participants and the total amount spent on participant compensation?
     \answerNA{}
\end{enumerate}

\end{enumerate}

%%%%%%%%%%%%%%%%%%%%%%%%%%%%%%%%%%%%%%%%%%%%%%%%%%%%%%%%%%%%
\newpage

\appendix

\section{Proofs of learning the whole environment}
\subsection{Proof of Theorem \ref{lemma:generic_bound}}\label{sec:proof_lemma:generic_bound}
\begin{proof}
We follow the standard information-theoretical regret decomposition:
\begin{equation*}
\begin{split}
     \BR_L(\pi_\ids) &= \sum_{\ell=1}^L\mathbb E\left[\mathbb E_\ell\left[V_{1,\pi^*}^{\cE}(s_1^\ell)-V_{1,\pi^\ell_{\ids}}^\cE(s_1^\ell)\right]\right] \\
     &= \sum_{\ell=1}^L\mathbb E\left[\sqrt{\frac{\left(\mathbb E_\ell\left[V_{1,\pi^*}^{\cE}(s_1^\ell)-V_{1,\pi^\ell_{\ids}}^\cE(s_1^\ell)\right]\right)^2}{\mathbb I_\ell^{\pi^\ell_{\ids}}(\cE; \cH_{\ell, H})}\mathbb I_\ell^{\pi^\ell_{\ids}}(\cE; \cH_{\ell, H})}\right]\\
     &\leq \sqrt{\mathbb E\left[\sum_{\ell=1}^L\Gamma_\ell(\pi^\ell_{\ids})\right]}\sqrt{\mathbb E\left[\sum_{\ell=1}^L\mathbb I_\ell^{\pi_{\ids}^\ell}(\cE; \cH_{\ell, H})\right]}= \sqrt{\mathbb I\left(\cE;\cD_{L+1}\right)\sum_{\ell=1}^L\mathbb E[\Gamma_\ell(\pi^\ell_{\ids})]}\,,
\end{split}
\end{equation*}
where the first inequality is from the Cauchy–Schwarz inequality and the last equation is due to the chain rule of mutual information
\begin{equation*}
    \mathbb I\left(\cE;\cD_{L+1}\right)=\mathbb I\left(\cE; \left(\cH_{1, H}, \ldots, \cH_{L, H}\right)\right) =  \sum_{\ell=1}^L\mathbb E\left[\mathbb I_{\ell}^{\pi_{\ids}^\ell}(\cE; \cH_{\ell,H})\right]\,.
\end{equation*}
According to the definition of $\Gamma^*$, we have $\Gamma_\ell(\pi^\ell_{\ids})\leq \Gamma^*$ for any $\ell\in[L]$. 
\end{proof}

\subsection{Proof of Lemma \ref{lemma:information_ratio_bound}}\label{sec:proof_lemma:information_ratio_bound}
\begin{proof}
From the definition of IDS policy stated in Eq.~\eqref{eqn:vanilla_IDS}, for any $\ell\in[L]$,
\begin{equation*}
   \Gamma_\ell(\pi^\ell_\ids)\leq \Gamma_\ell(\pi^\ell_{\ts})= \frac{\mathbb E_\ell\left[V_{1,\pi^*}^\cE(s_1^\ell)-V_{1,\pi^\ell_{\ts}}^\cE(s_1^\ell)\right]^2}{\mathbb I_\ell^{\pi_{\ts}^\ell}\left(\cE;\cH_{\ell, H}\right)}\,.
\end{equation*}
We first decompose the one-step regret as follows,
 \begin{equation}\label{eqn:regret_decom}
     \begin{split}
          &\mathbb E_\ell\left[V_{1,\pi^*}^\cE(s_1^\ell)-V_{1,\pi^\ell_{\ts}}^\cE(s_1^\ell)\right]= \underbrace{\mathbb E_\ell\left[V_{1,\pi^*}^\cE(s_1^\ell)-V_{1,\pi^\ell_{\ts}}^{\bar{\cE}_\ell}(s_1^\ell)\right]}_{I_1} +  \underbrace{\mathbb E_\ell\left[V_{1,\pi^\ell_{\ts}}^{\bar{\cE}_\ell}(s_1^\ell)-V_{1,\pi^\ell_{\ts}}^\cE(s_1^\ell)\right]}_{I_2}\,.
     \end{split}
 \end{equation}

\paragraph{Bounding $I_1$} Note that conditional on $\cD_\ell$, the law of $\pi^{\ell}_{\ts}$ is the same as the law of $\pi^*$ and both $\pi^*$ and $\pi^\ell_{\ts}$ are independent of $\bar \cE_\ell$. %because $\pi^\ell$ is the function of a sample drawn from the posterior distribution of $\cE$ while $\pi^*$ is the function of $\cE$. Thus, we have 
This fact implies that
\begin{equation*}
\begin{split}
     I_1 &= \mathbb E_\ell\left[V_{1,\pi^*}^{ \cE}(s_1^\ell)-V_{1,\pi^{*}}^{\bar\cE_\ell}(s_1^\ell)\right]\,.
\end{split}
\end{equation*}
Denote
\begin{equation*}
    \Delta_h^{\cE}(s,a) =\mathbb E_{s'\sim P_h^{ \cE}(\cdot|s,a)}[V_{h+1,\pi^*}^{ \cE}(s')]-\mathbb E_{s'\sim P_h^{\bar \cE_\ell}(\cdot|s,a)}[V_{h+1,\pi^*}^{ \cE}(s')]\,.
\end{equation*}
Applying Lemma \ref{lemma:bellman_residual}, we have 
\begin{equation*}
\begin{split}
    I_1
      &=\mathbb E_{\ell}\left[\sum_{h=1}^H\mathbb E_{\pi^*}^{\bar \cE_\ell}\left[\Delta_h^{\cE}(s_h^\ell,a_h^\ell)\right]\right]=\mathbb E_{\ell}\left[\sum_{h=1}^H\mathbb E\left[\mathbb E_{\pi^*}^{\bar \cE_\ell}\left[\Delta_h^{\cE}(s_h^\ell,a_h^\ell)\right]\big| \cE\right]\right]\\
     &=
      \sum_{h=1}^H\mathbb E_{\ell}\left[\mathbb E\left[\sum_{(s,a)}d_{h,\pi^*}^{\bar\cE_\ell}(s,a)\Delta_h^{\cE}(s,a)\Big|\cE\right]\right]\,,
      \\
       &=
      \sum_{h=1}^H\mathbb E_{\ell}\left[\sum_{(s,a)}d_{h,\pi^*}^{\bar\cE_\ell}(s,a)\Delta_h^{\cE}(s,a)\right]\,,
      \\
        &=\sum_{h=1}^H\mathbb E_{\ell}\left[ \sum_{(s,a)}\frac{d_{h,\pi^*}^{\bar \cE_\ell}(s,a)}{(\mathbb E_{\ell}[d_{h,\pi^*}^{\bar \cE_\ell}(s,a)])^{1/2}}(\mathbb E_{\ell}[d_{h,\pi^*}^{\bar \cE_\ell}(s,a)])^{1/2}\Delta_h^{\cE}(s,a)\right]\,.
\end{split}
\end{equation*}
Here we just divide the element that corresponds to $d_{h,\pi^*}^{\bar \cE_\ell}(s,a)>0$.

Applying Cauchy–Schwarz inequality and using the fact that $d_{h,\pi^*}^{\bar\cE_\ell}(s,a)\leq 1$, we have 
\begin{equation}\label{eqn:bound_state_occupancy}
\begin{split}
    I_1
     &\leq \left(\sum_{h=1}^H\mathbb E_{\ell}\left[\sum_{(s,a)}\frac{(d_{h,\pi^*}^{\bar\cE_\ell}(s,a))^2}{\mathbb E_{\ell}[d_{h,\pi^*}^{\bar\cE_\ell}(s,a)]}\right]\right)^{1/2}\left(\sum_{h=1}^H\mathbb E_{\ell}\left[\sum_{(s,a)}\mathbb E_{\ell}[d_{h,\pi^*}^{\bar\cE_\ell}(s,a)](\Delta_h^{\cE}(s,a))^2\right]\right)^{1/2}\\
     &\leq \left(\sum_{h=1}^H\mathbb E_{\ell}\left[\sum_{(s,a)}\frac{d_{h,\pi^*}^{\bar\cE_\ell}(s,a)}{\mathbb E_{\ell}[d_{h,\pi^*}^{\bar\cE_\ell}(s,a)]}\right]\right)^{1/2}\left(\sum_{h=1}^H\mathbb E_{\ell}\left[\sum_{(s,a)}\mathbb E_{\ell}[d_{h,\pi^*}^{\bar\cE_\ell}(s,a)](\Delta_h^{\cE}(s,a))^2\right]\right)^{1/2}\,.
    \end{split}
\end{equation}
First, note that 
\begin{equation}\label{eqn:I11}
    \sum_{h=1}^H\mathbb E_{\ell}\left[\sum_{(s,a)}\frac{d_{h,\pi^*}^{\bar\cE_\ell}(s,a)}{\mathbb E_{\ell}[d_{h,\pi^*}^{\bar\cE_\ell}(s,a)]}\right]=\sum_{h=1}^H\left[\sum_{(s,a)}\frac{\mathbb E_{\ell}[d_{h,\pi^*}^{\bar\cE_\ell}(s,a)]}{\mathbb E_{\ell}[d_{h,\pi^*}^{\bar\cE_\ell}(s,a)]}\right] \leq HSA \,.
\end{equation}
Second, since the law of $\pi^*$ is the same as $\pi^\ell_{\ts}$ conditional on $\cD_\ell$, we have 
\begin{equation*}
    \mathbb E_\ell\left[d^{\bar \cE_\ell}_{h,\pi^*}(s, a)\right]=\mathbb E_\ell\left[d_{h,\pi^\ell_{\ts}}^{\bar \cE_\ell}(s,a)\right]\,.
\end{equation*}
Then we can obtain 
\begin{equation*}
    \begin{split}
       \mathbb E_{\ell}\left[\sum_{(s,a)}\mathbb E_{\ell}[d_{h,\pi^*}^{\bar\cE_\ell}(s,a)](\Delta_h^{\cE}(s,a))^2\right]&=\mathbb E_{\ell}\left[\sum_{(s,a)}\mathbb E_{\ell}\left[d_{h,\pi_\ts^\ell}^{\bar\cE_\ell}(s,a)\right](\Delta_h^{\cE}(s,a))^2\right]\\
       &=\mathbb E_{\ell}\left[\sum_{(s,a)}d_{h,\pi_\ts^\ell}^{\bar\cE_\ell}(s,a)\right]\mathbb E_{\ell}\left[(\Delta_h^{\cE}(s,a))^2\right]\,.
           \end{split}
\end{equation*}
Given $\cD_\ell$, we have $d_{h,\pi_\ts^\ell}^{\bar\cE_\ell}(s,a)$ and $\Delta_h^{\cE}(s,a)$ are independent. This implies
    
\begin{equation}\label{eqn:I12}
    \begin{split}
     &\mathbb E_{\ell}\left[\sum_{(s,a)}d_{h,\pi_\ts^\ell}^{\bar\cE_\ell}(s,a)\right]\mathbb E_{\ell}\left[(\Delta_h^{\cE}(s,a))^2\right]=\mathbb E_{\ell}\left[\sum_{(s,a)}d_{h,\pi_\ts^\ell}^{\bar\cE_\ell}(s,a)(\Delta_h^{\cE}(s,a))^2\right]\\
       &=\mathbb E_{\ell}\left[\mathbb E_{\pi^\ell_{\ts}}^{\bar\cE_\ell}\left[(\Delta_h^{ \cE}(s_h^\ell,a_h^\ell))^2\right]\right]\\
        &=H^2\mathbb E_{\ell}\left[\mathbb E_{\pi^\ell_{\ts}}^{\bar \cE_\ell}\left[\left(\mathbb E_{s'\sim P_h^{\cE}(\cdot|s_h^\ell,a_h^\ell)}\left[V_{h+1,\pi^*}^{ \cE}(s')/H\right]-\mathbb E_{s'\sim P_h^{\bar \cE_\ell}(\cdot|s_h^\ell,a_h^\ell)}\left[V_{h+1,\pi^*}^{ \cE}(s')/H\right]\right)^2\right]\right]\,,
    \end{split}
\end{equation}
where $\mathbb E_{\pi^\ell_{\ts}}^{\bar \cE_\ell}$ is taken with respect to $s_h^\ell, a_h^\ell$ and $\mathbb E_\ell$ is taken with respect to $\pi^\ell_{\ts}$ and $\cE$.

Combining Eqs.~\eqref{eqn:bound_state_occupancy}-\eqref{eqn:I12} together and applying Lemma \ref{lemma:pinsker} which is a variant of Pinsker's inequality,
\begin{equation}\label{eqn:I_1_bound}
\begin{split}
    I_1&\leq \sqrt{SAH^3}\left(\sum_{h=1}^H\mathbb E_{\ell}\mathbb E_{\pi^\ell_{\ts}}^{\bar \cE_\ell}\left[\frac{1}{2}D_{\KL}\left(P_h^{\cE}(\cdot|s_h^\ell,a_h^\ell)||P_h^{\bar \cE_\ell}(\cdot|s_h^\ell,a_h^\ell)\right)\right]\right)^{1/2}\,.
    \end{split}
\end{equation}

The next lemma establishes the relationship between the cumulative KL-distance in Eq.~\eqref{eqn:I_1_bound} and mutual information $\mathbb I_\ell^{\pi_{\ts}^\ell}(\cE; \cH_{\ell, H})$. This lemma will be frequently used later. 
\begin{lemma}\label{lemma:information_gain_eqn}
For environment $\cE$ and its corresponding mean of posterior measure $\bar \cE_\ell$, the following holds for any policy $\pi$,
\begin{equation*}
   \sum_{h=1}^H \mathbb E_{\ell}\mathbb E_{\pi}^{\bar \cE_\ell}\left[D_{\KL}\left(P_h^{\cE}(\cdot|s_h^\ell,a_h^\ell)||P_h^{\bar \cE_\ell}(\cdot|s_h^\ell,a_h^\ell)\right)\right]= \mathbb I_\ell^{\pi}\left(\cE; \cH_{\ell, H}\right)\,.
\end{equation*}
\end{lemma}

Together with Eq.~\eqref{eqn:I_1_bound},
\begin{equation}\label{eqn:I1_upper}
\begin{split}
    I_1&\leq \sqrt{\frac{1}{2}SAH^3\mathbb I_\ell^{\pi_{\ts}^\ell}\left(\cE; \cH_{\ell, H}\right)}\,.
    \end{split}
\end{equation}

\paragraph{Bounding $I_2$} The upper bound for $I_2$ is similar. Applying Lemma \ref{lemma:bellman_residual} and repeating the steps of showing Eq.~\eqref{eqn:I_1_bound},
\begin{equation*}
\begin{split}
   & \mathbb E_\ell\left[V_{1,\pi^\ell_{\ts}}^\cE(s_1^\ell)-V_{1,\pi^\ell_{\ts}}^{\bar{\cE}_\ell}(s_1^\ell)\right]\\
   &=\mathbb E_{\ell}\left[\sum_{h=1}^H\mathbb E_{\pi^\ell_{\ts}}^{\bar \cE_\ell}\left[\mathbb E_{s'\sim P_h^{ \cE}(\cdot|s_h^\ell,a_h^\ell)}[V_{h+1,\pi^\ell_{\ts}}^{ \cE}(s')]-\mathbb E_{s'\sim P_h^{\bar \cE_\ell}(\cdot|s_h^\ell,a_h^\ell)}[V_{h+1,\pi^\ell_{\ts}}^{ \cE}(s')]\right]\right]\\
    &\leq \sqrt{SAH^3}\left(\sum_{h=1}^H\mathbb E_{\ell}\mathbb E_{\pi^\ell_{\ts}}^{\bar \cE_\ell}\left[\frac{1}{2}D_{\KL}\left(P_h^{\cE}(\cdot|s_h^\ell,a_h^\ell)||P_h^{\bar \cE_\ell}(\cdot|s_h,a_h)\right)\right]\right)^{1/2}\,.
  \end{split}
\end{equation*}
Applying Lemma \ref{lemma:information_gain_eqn} again, we obtain
\begin{equation}\label{eqn:I2_upper}
   I_2\leq \sqrt{SAH^3\mathbb I_\ell^{\pi_{\ts}^\ell}\left(\cE; \cH_{\ell, H}\right)}
\end{equation}

Combining Eqs.~\eqref{eqn:regret_decom}, \eqref{eqn:I1_upper} and \eqref{eqn:I2_upper} together, we have
\begin{equation*}
  \Gamma^*\leq \max_{\ell\in[L]}\Gamma_\ell(\pi^\ell_\ts) = \frac{\left(\mathbb E_\ell\left[V_{1,\pi^\ell_{\ts}}^\cE(s_1^\ell)-V_{1,\pi^*}^\cE(s_1^\ell)\right]\right)^2}{\mathbb I_\ell^{\pi_{\ts}^\ell}\left(\cE; \cH_{\ell, H}\right)}\leq 2SAH^3\,.
\end{equation*}
This ends the proof.
\end{proof}

\subsection{Proof of Lemma \ref{lemma:information_gain}}\label{sec:proof_lemma:information_gain}
\begin{proof}
Now we start to bound $\mathbb I\left(\cE;\cD_{L+1}\right)$. Assume the prior probability measure of $\cE$ is $\rho(\cE)$ that takes the value in $\Theta$ and suppose the Bayes mixture density $p_{\rho}(\cD_{L+1}) = \int_{\cE\in\Theta} p(\cD_{L+1}|\cE)d \rho(\cE)$. Let $\{\Theta_k\}_{k=1}^K$ be a partition of $\Theta$.
We choose ${\rho}_1$ as an uniform distribution over $\{\Theta_k\}_{k=1}^K$ such that $q(\cD_{L+1}) = p_{\rho_1}(\cD_{L+1}) = \int_{\theta\in\Theta} p(\cD_{L+1}|\theta)\d \rho_1(\theta)$ and we denote $\mathbb Q_{\cD_{L+1}}$ as the corresponding probability measure. Based on the equivalent form of mutual information,
\begin{equation}\label{eqn:mutual_information}
\begin{split}
    \mathbb I(\cE; \cD_{L+1}) &=\mathbb E_{\cE}\left[D_{\KL}(\mathbb P_{\cD_{L+1}|\cE}||\mathbb P_{\cD_{L+1}})\right]\\
    &= \int_{\cE\in\Theta}\int p(\cD_{L+1}|\cE)\log\Big(\frac{p(\cD_{L+1}|\cE)}{p_\rho(\cD_{L+1})}\Big)\mu(\d\cD_{L+1})\d\rho(\cE)\\
    &\leq  \int_{\cE\in\Theta}\int p(\cD_{L+1}|\cE)\log\Big(\frac{p(\cD_{L+1}|\cE)}{q(\cD_{L+1})}\Big)\mu(\d\cD_{L+1})\d\rho(\cE)\\
    &=\int_{\cE\in\Theta} D_{\KL}(\mathbb P_{\cD_{L+1}|\cE}||\mathbb Q_{\cD_{L+1}})\d\rho(\cE)\,.
    \end{split}
\end{equation}
where the inequality is due to the fact that Bayes mixture density $p_{\rho}(\cD_{L+1})$ minimizes the average KL divergences over any choice of densities $q(\cD_{L+1})$.  

According to the definition of the KL-divergence term,
\begin{equation}\label{eqn:KL_calculation}
\begin{split}
    D_{\KL}\left(\mathbb P_{\cD_{L+1}|\cE}||\mathbb Q_{\cD_{L+1}}\right) &= \mathbb E\left[\log \frac{p(\cD_{L+1}|\cE)}{1/K\sum_{\tilde{\cE}\in \Theta_k}p(\cD_{L+1}|\tilde{\cE})}\right]\\
    &\leq \mathbb E\left[\log \frac{p(\cD_{L+1}|\cE)}{1/Kp(\cD_{L+1}|\tilde{\cE})}\right]\\
    &\leq \log (K) + D_{\KL}\left(\mathbb P_{\cD_{L+1}|\cE}|| \mathbb P_{\cD_{L+1}|\tilde{\cE}}\right)\,.
\end{split}
\end{equation}
By the chain rule of KL-divergence,
\begin{equation*}
\begin{split}
   D_{\KL}\left(\mathbb P_{\cD_{L+1}|\cE}||\mathbb P_{\cD_{L+1}|\tilde{\cE}}\right)&= \mathbb E\left[\sum_{\ell=1}^L\sum_{h=1}^HD_{\KL}\left(P(\cdot|s_h^\ell,a_h^\ell,\cE)||P(\cdot|s_h^\ell, a_h^\ell, \tilde \cE)\right)\right]\\
   &=\mathbb E\left[\sum_{\ell=1}^L\sum_{h=1}^HD_{\KL}\left(P_h^\cE(\cdot|s_h^\ell,a_h^\ell)||P_h^{\tilde \cE}(\cdot|s_h^\ell, a_h^\ell)\right)\right]\,.
   \end{split}
\end{equation*}

It remains to construct a partition of $\{\Theta_k\}_{k=1}^K$ such that for any $\cE\in\Theta$, there always exists $\tilde \cE\in\Theta_k$ such that for any $s, a, h$, 
\begin{equation*}
    D_{\KL}\left(P_h^\cE(\cdot|s,a)||P_h^{\tilde \cE}(\cdot|s,a)\right)\leq \varepsilon\,.
\end{equation*}
The existence of such partition is shown in the following theorem.

\begin{theorem}[Divergence covering number]\label{thm:divergence_cover}
For any $0<\varepsilon<1$, suppose $ \cN(k, \varepsilon)$ is a divergence covering set over $k$-dimensional probability simplex such that for any $P$, there exists $\tilde P\in  \cN(k, \varepsilon)$ such that
\begin{equation*}
    D_{\KL}(P||\tilde P)\leq \varepsilon\,.
\end{equation*}
Then there exists such set whose covering number can be bounded by
\begin{equation*}
     |\cN(k, \varepsilon)|\leq 8^{k-1}\left(\frac{k-1}{\varepsilon}\right)^{\frac{k-1}{2}}\,.
\end{equation*}
\end{theorem}
The proof could be found in \citet[Theorem 4]{tang2021capacity}. For each $s, a, h$, we construct such divergence covering set $\cN_{sa}^h(S, \varepsilon)$ and the partition $\{\Theta_k\}_{k=1}^K$ is constructed such that $\cE\in\Theta_k$ if for any $s, a, h$, $P_h^{\cE}(\cdot|s,a)\in\cN_{sa}^h(S,\varepsilon)$. According to Theorem \ref{thm:divergence_cover},
\begin{equation*}
    \log(K)\leq \log\left(|\cN_{sa}^h(S,\varepsilon)|^{SAH}\right)\leq (S-1)SAH\log(8)+(S-1)SAH\log\left(\frac{S-1}{\varepsilon}\right)\,.
\end{equation*}
Therefore, together with Eq.~\eqref{eqn:KL_calculation},
\begin{equation*}
     \mathbb I(\cE; \cD_{L+1})\leq LH\varepsilon+S^2AH\log\left(S/\varepsilon\right)\,.
\end{equation*}
By choosing $\varepsilon=1/LH$,
\begin{equation*}
     \mathbb I(\cE; \cD_{L+1})\leq 2S^2AH\log\left(SLH\right)\,.
\end{equation*}
This ends the proof.
\end{proof}

\subsection{Proof of Proposition \ref{prop:augmented_reward}}\label{sec:proof_prop:augmented_reward}
According to the linearity of expectation and independence of priors over different layers, we can obtain $\mathbb E_{\ell}[V_{1,\pi}^\cE(s_1^\ell)]=V_{1,\pi}^{\bar \cE_\ell}(s_1^\ell)$ which implies
\begin{equation*}
   \mathbb E_{\ell}[V_{1,\pi}^\cE(s_1^\ell)]= \mathbb E_{\pi}^{\bar\cE_\ell}\left[\sum_{h=1}^H r_h(s_h, a_h)\right]\,.
\end{equation*}
According to the proof of Lemma \ref{lemma:information_gain_eqn} in Appendix \ref{sec:proof_lemma:information_gain_eqn} \begin{equation*}
    \begin{split}
         \mathbb I_\ell^{\pi}\left(\cE; \cH_{\ell, H}\right)=\sum_{h=1}^H\mathbb E_{\pi}^{\bar\cE_\ell}\left[ \int D_{\KL}\left(P_h^{\cE}(\cdot|s_{h},a_{h})||P_h^{\bar \cE_\ell}(\cdot|s_{h},a_{h})\right)\d \mathbb P_\ell(\cE)\right]\,.
    \end{split}
\end{equation*}
Combining them together,
\begin{equation*}
    \begin{split}
        &\mathbb E_{\ell}[V_{1,\pi}^\cE(s_1^\ell)]+\lambda  \mathbb I_\ell^{\pi}\left(\cE; \cH_{\ell, H}\right) \\
        =&\mathbb E_{\pi}^{\bar\cE_\ell}\left[\sum_{h=1}^H\left( r_h(s_h, a_h)+\lambda \int D_{\KL}\left(P_h^{\cE}(\cdot|s_{h},a_{h})||P_h^{\bar \cE_\ell}(\cdot|s_{h},a_{h})\right)\d \mathbb P_\ell(\cE)\right)\right]\,.
    \end{split}
\end{equation*}

\subsection{Proof of Theorem \ref{thm:a-ids}}\label{sec:proof_thm:a-ids}
\begin{proof}
Using the fact that $2ab\leq a^2+b^2$, we have for any policy $\pi$
\begin{equation*}
\begin{split}
   & \frac{ \mathbb E_\ell\left[V_{1,\pi^*}^\cE(s_1^\ell)-V_{1,\pi}^\cE(s_1^\ell)\right]\sqrt{\lambda\mathbb I_\ell^{\pi}\left(\cE; \cH_{\ell, H}\right)}}{\sqrt{\lambda\mathbb I_\ell^{\pi}\left(\cE; \cH_{\ell, H}\right)}}
   \leq \frac{\left(\mathbb E_\ell\left[V_{1,\pi^*}^\cE(s_1^\ell)-V_{1,\pi}^\cE(s_1^\ell)\right]\right)^2}{2\lambda\mathbb I_\ell^{\pi}\left(\cE; \cH_{\ell, H}\right)}+\frac{\lambda}{2} \mathbb I_\ell^{\pi}\left(\cE; \cH_{\ell, H}\right)\,.
    \end{split}
\end{equation*}
This implies 
\begin{equation}\label{eqn:r1}
    \begin{split}
       & \mathbb E_\ell\left[V_{1,\pi^*}^\cE(s_1^\ell)-V_{1,\pi}^\cE(s_1^\ell)\right]-\frac{\lambda}{2} \mathbb I_\ell^{\pi}\left(\cE; \cH_{\ell, H}\right)
       \leq \frac{\left(\mathbb E_\ell\left[V_{1,\pi^*}^\cE(s_1^\ell)-V_{1,\pi}^\cE(s_1^\ell)\right]\right)^2}{2\lambda\mathbb I_\ell^{\pi}\left(\cE; \cH_{\ell, H}\right)}\,.
    \end{split}
\end{equation}
We follow the regret decomposition as 
\begin{equation*}
    \begin{split}
        &\BR_L(\pi_{\text{r-IDS}}) \\
        &= \sum_{\ell=1}^L\mathbb E\left[\mathbb E_\ell\left[V_{1,\pi^*}^{\cE}(s_1^\ell)-V_{1,\pi^\ell_{\text{r-IDS}}}^\cE(s_1^\ell)\right]\right]\\
        &= \mathbb E\left[\sum_{\ell=1}^L\mathbb E_\ell\left[V_{1,\pi^*}^{\cE}(s_1^\ell)-V_{1,\pi^\ell_{\text{r-IDS}}}^\cE(s_1^\ell)\right] -\lambda\sum_{\ell=1}^L\mathbb I_\ell^{_{\text{r-IDS}}}\left(\cE; \cH_{\ell, H}\right)\right]+\mathbb E\left[\lambda\sum_{\ell=1}^L\mathbb I_\ell^{_{\text{r-IDS}}}\left(\cE; \cH_{\ell, H}\right)\right]\,.
            \end{split}
\end{equation*}
From the definition of $\pi^\ell_{\text{r-IDS}}$ and Eq.~\eqref{eqn:r1}, we have 
\begin{equation*}
\begin{split}
     \mathbb E_\ell\left[V_{1,\pi^*}^\cE(s_1^\ell)-V_{1,\pi^\ell_{\text{r-IDS}}}^\cE(s_1^\ell)\right]-\frac{\lambda}{2} \mathbb I_\ell^{\pi^\ell_{\text{r-IDS}}}\left(\cE; \cH_{\ell, H}\right)&\leq  \mathbb E_\ell\left[V_{1,\pi^*}^\cE(s_1^\ell)-V_{1,\pi^\ell_{\text{IDS}}}^\cE(s_1^\ell)\right]-\frac{\lambda}{2} \mathbb I_\ell^{\pi^\ell_{\text{IDS}}}\left(\cE; \cH_{\ell, H}\right)\\
     &\leq\frac{\left(\mathbb E_\ell\left[V_{1,\pi^*}^\cE(s_1^\ell)-V_{1,\pi^\ell_{\text{IDS}}}^\cE(s_1^\ell)\right]\right)^2}{2\lambda\mathbb I_\ell^{\pi^\ell_{\text{IDS}}}\left(\cE; \cH_{\ell, H}\right)}\leq \frac{\Gamma^* }{2\lambda}\,,
     \end{split}
\end{equation*}
where $\Gamma^*$ is the worst-case information ratio. Overall, this implies
\begin{equation*}
    \begin{split}
    \BR_L(\pi_{\text{r-IDS}})     
        &\leq \frac{L\mathbb E[\Gamma^*] }{2\lambda}+\lambda \mathbb I\left(\cE;\cD_{L+1}\right)\,.
    \end{split}
\end{equation*}
Letting $\lambda = \sqrt{L\mathbb E[\Gamma^*]/\mathbb I(\cE;\cD_{L+1})}$, we have 
\begin{equation*}
     \BR_L(\pi^{\text{r-IDS}}) \leq \sqrt{\frac{3}{2}L\mathbb E[\Gamma^*]\mathbb I(\cE;\cD_{L+1})}\,.
\end{equation*}
This ends the proof.
\end{proof}

\section{Proofs of learning a surrogate environment}

\subsection{Proof of Lemma \ref{lemma:surrogate_learning}}
\begin{proof}
Conditional on $\cD_\ell$, let  $\cE_\ell$ be an independent sample of $\cE$. Note that 
\begin{equation*}
\begin{split}
    \mathbb E_\ell\left[V_{1,\pi^*_\cE}^{\cE_\ell}(s_1^\ell)\big|\cE_\ell\in\Theta_k\right] &= \sum_{\cE\in\Theta_k^{\varepsilon}}\mathbb P\left(\cE_\ell=\cE|\cE_\ell\in\Theta_k\right)\mathbb E_\ell\left[V_{1,\pi^*_\cE}^{\cE}(s_1^\ell)\big|\cE_\ell\in\Theta_k\right]\\
    &=\sum_{\cE\in\Theta_k^{\varepsilon}}\mathbb P\left(\cE_\ell=\cE|\cE_\ell\in\Theta_k\right)\mathbb E_\ell\left[V_{1,\pi^*_\cE}^{\cE}(s_1^\ell)\right]\,,
\end{split}
\end{equation*}
where the last equation is due to the independence between $\cE_\ell$ and $\cE$.
For each $k\in[K]$, according to Lemma \ref{lemma1}, there exists $\cE_1^{k,\ell}, \cE_2^{k,\ell}\in\Theta_k^{\varepsilon}$ and $r_{k,\ell}\in[0,1]$ such that 
\begin{equation*}
\begin{split}
    r_{k,\ell}\mathbb E_\ell\left[V_{1,\pi^*_\cE}^{\cE_1^{k,\ell}}(s_1^\ell)\right]+ (1-r_{k,\ell})\mathbb E_\ell\left[V_{1,\pi^*_\cE}^{\cE_2^{k,\ell}}(s_1^\ell)\right]&\leq \sum_{\cE\in\Theta_k^{\varepsilon}}\mathbb P\left(\cE_\ell=\cE|\cE_\ell\in\Theta_k\right)\mathbb E_\ell\left[V_{1,\pi^*_\cE}^{\cE}(s_1^\ell)\right]\\
    &=\mathbb E_\ell\left[V_{1,\pi^*_\cE}^{\cE_\ell}(s_1^\ell)\big| \cE_\ell\in\Theta_k^{\varepsilon}\right]\,.
    \end{split}
\end{equation*}
The surrogate learning target $\tilde \cE_\ell^*$ is a random variable such that
\begin{equation}\label{eqn:surrogate_env}
    \mathbb P_\ell\left(\tilde \cE_\ell^*=\cE_1^{k,\ell}\big|\zeta=k\right)=r_{k,\ell}, \mathbb P_\ell\left(\tilde \cE_\ell^*=\cE_2^{k,\ell}\big|\zeta=k\right)=1-r_{k,\ell}\,.
\end{equation}
This implies the law of $\tilde \cE_\ell^*$ only depends on $\zeta$ and conditional on $\zeta$, $\tilde \cE_\ell^*$ is independent of $\cE$.

We also need to ensure learning towards $\tilde \cE_\ell^*$ will not occur too much additional regret. Let $\tilde \cE_\ell$ be an independent sample of $\tilde \cE_\ell^*$. From the law of total expectations,
\begin{equation*}
\begin{split}
    &\mathbb E_\ell\left[V_{1,\pi^*_\cE}^{\tilde \cE_\ell}(s_1^\ell)-V_{1,\pi^*_\cE}^{\cE_\ell}(s_1^\ell)\right] \\
    &= \sum_{k=1}^K\mathbb P\left(\cE_\ell\in\Theta_k^{\varepsilon}\right)\mathbb E_\ell\left[V_{1,\pi^*_\cE}^{\tilde \cE_\ell}(s_1^\ell)-V_{1,\pi^*_\cE}^{ \cE_\ell}(s_1^\ell)|\cE_\ell\in\Theta_k^{\varepsilon}\right]\\
    &=\sum_{k=1}^K \mathbb P\left(\cE_\ell\in\Theta_k^{\varepsilon}\right)\left(  r_{k,\ell}\mathbb E_\ell\left[V_{1,\pi^*_\cE}^{\cE_1^{k,\ell}}(s_1^\ell)\right]+ (1-r_{k,\ell})\mathbb E_\ell\left[V_{1,\pi^*_\cE}^{\cE_2^{k,\ell}}(s_1^\ell)\right]-\mathbb E_\ell\left[V_{1,\pi^*_\cE}^{\cE_\ell}(s_1^\ell)\big| \cE_\ell\in\Theta_k^{\varepsilon}\right]\right)\\
    &\leq 0\,.
    \end{split}
\end{equation*}
On the other hand, we have $\mathbb E_\ell[V_{1,\pi^*_\cE}^{\cE_\ell}(s_1^\ell)]=\mathbb E_\ell[V_{1, \pi^\ell_{\ts}}^\cE(s_1^\ell)]$ and this implies
\begin{equation}\label{eqn:e1}
     \mathbb E_\ell\left[V_{1,\pi^*_\cE}^{\tilde \cE_\ell}(s_1^\ell)-V_{1, \pi^\ell_{\ts}}^\cE(s_1^\ell)\right]\leq 0\,.
\end{equation}
When $\cE\in\Theta_k^{\varepsilon}$, then $\zeta=k$ which implies $\tilde \cE_\ell^*\in\Theta_k^{\varepsilon}$ either. That means $\cE$ and $\tilde \cE^*$ are in the same partition and
\begin{equation}\label{eqn:e2}
    \mathbb E_\ell\left[V_{1,\pi^*_\cE}^\cE(s_1^\ell)-V_{1,\pi^*_\cE}^{\tilde \cE_\ell^*}(s_1^\ell)\right]\leq \varepsilon\,.
\end{equation}
Putting Eqs.~\eqref{eqn:e1}-\eqref{eqn:e2} together,
\begin{equation}\label{eqn:additional_regret}
\begin{split}
    & \mathbb E_{\ell}\left[V_{1, \pi^*_\cE}^\cE(s_1^\ell)-V_{1, \pi^\ell_{\ts}}^\cE(s_1^\ell)\right]-\mathbb E_{\ell}\left[V_{1,\pi^*_\cE}^{\tilde \cE_\ell^*}(s_1^\ell)-V_{1,\pi^*_\cE}^{\tilde \cE_\ell}(s_1^\ell)\right]\leq \varepsilon\,.
\end{split}
\end{equation}
Noticing that $\mathbb E_\ell[V_{1,\pi^*_\cE}^{\tilde \cE_\ell}(s_1^\ell)]=\mathbb E_\ell[V_{1,\pi^\ell_{\ts}}^{\tilde \cE_\ell^*}(s_1^\ell)]$, this ends the proof.
\end{proof}

\subsection{Proof of Theorem \ref{lemma:IDS+}}\label{sec:proof_lemma:IDS+}
\begin{proof}
We decompose 
\begin{equation*}
    \begin{split}
          \BR_L(\pi_{\sids}) &= \sum_{\ell=1}^L\mathbb E\left[\mathbb E_\ell\left[V_{1,\pi^*}^{\cE}(s_1^\ell)-V_{1,\pi^\ell_{\sids}}^\cE(s_1^\ell)\right]-\varepsilon\right]+L\varepsilon \\
     &\leq \sqrt{\mathbb E\left[\sum_{\ell=1}^L\frac{\left(\mathbb E_\ell\left[V_{1,\pi^*}^{\cE}(s_1^\ell)-V_{1,\pi^\ell_{\sids}}^\cE(s_1^\ell)\right]-\varepsilon\right)^2}{\mathbb I_\ell^{\pi^\ell_{\sids}}(\tilde \cE_\ell^*; \cH_{\ell, H})}\right]}\sqrt{\mathbb E\left[\sum_{\ell=1}^L\mathbb I_\ell^{\pi^\ell_{\sids}}(\tilde \cE_\ell^*; \cH_{\ell, H})\right]}+L\varepsilon\,.
    \end{split}
\end{equation*}
From the definition of $ \pi^\ell_{\sids}$ in Eq.~\eqref{def:information_ratio_ids+},
\begin{equation*}
    \begin{split}
          \BR_L(\pi_{\sids}) 
      \leq \sqrt{\mathbb E\left[\sum_{\ell=1}^L\frac{\left(\mathbb E_\ell\left[V_{1,\pi^*}^{\cE}(s_1^\ell)-V_{1,\pi^\ell_{\ts}}^\cE(s_1^\ell)\right]-\varepsilon\right)^2}{\mathbb I_\ell^{\pi^\ell_{\ts}}(\tilde \cE_\ell^*; \cH_{\ell, H})}\right]}\sqrt{\mathbb E\left[\sum_{\ell=1}^L\mathbb I_\ell^{\pi^\ell_{\text{r-IDS}}}(\tilde \cE_\ell^*; \cH_{\ell, H})\right]}+L\varepsilon\,.
    \end{split}
\end{equation*}

According to Lemma \ref{lemma:surrogate_learning},
\begin{equation*}
    \mathbb E_\ell\left[V_{1,\pi^*}^{\cE}(s_1^\ell)-V_{1,\pi^\ell_{\ts}}^\cE(s_1^\ell)\right]-\varepsilon\leq \mathbb E_{\ell}\left[V_{1,\pi^*}^{\tilde \cE_\ell^*}(s_1^\ell)-V_{1,\pi^\ell_{\ts}}^{\tilde \cE_\ell^*}(s_1^\ell)\right]\,.
\end{equation*}
For any $\ell\in[L]$, conditional on $\zeta$, we have $\tilde \cE_\ell^*$ and $\cH_{\ell, H}$ are independent under the law of $\mathbb P_{\ell, \pi^\ell_{\sids}}$. By the data processing inequality, we have 
\begin{equation*}
    \mathbb I_\ell^{\pi^\ell_{\sids}}(\tilde \cE_\ell^*; \cH_{\ell, H}) \leq \mathbb I_\ell^{\pi^\ell_{\sids}}(\zeta; \cH_{\ell, H})\,.
\end{equation*}
Therefore, 
\begin{equation*}
     \begin{split}
           \BR_L(\pi_{\sids})& \leq \sqrt{\mathbb E\left[\sum_{\ell=1}^L\frac{\left( \mathbb E_{\ell}\left[V_{1,\pi^*}^{\tilde \cE_\ell^*}(s_1^\ell)-V_{1,\pi^\ell_{\ts}}^{\tilde \cE_\ell^*}(s_1^\ell)\right]\right)^2}{\mathbb I_\ell^{\pi^\ell_{\ts}}(\tilde \cE_\ell^*; \cH_{\ell, H})}\right]}\sqrt{\mathbb E\left[\sum_{\ell=1}^L\mathbb I_\ell^{\pi^\ell_{\sids}}(\zeta; \cH_{\ell, H})\right]}+L\varepsilon\\
          &= \sqrt{\mathbb E\left[\sum_{\ell=1}^L\frac{\left( \mathbb E_{\ell}\left[V_{1,\pi^*}^{\tilde \cE_\ell^*}(s_1^\ell)-V_{1,\pi^\ell_{\ts}}^{\tilde \cE_\ell^*}(s_1^\ell)\right]\right)^2}{\mathbb I_\ell^{\pi^\ell_{\ts}}(\tilde \cE_\ell^*; \cH_{\ell, H})}\right]}\sqrt{\mathbb I(\zeta; \cD_{L+1})}+L\varepsilon\,.
     \end{split}
\end{equation*}
This ends the proof.
\end{proof}

\subsection{Proof of Lemma \ref{lemma:cover_value}}\label{sec:proof_lemma:cover_value}
\begin{proof}
We construct a partition over $\Theta$ such that Eq.~\eqref{eqn:cover} holds. For any $\cE_1, \cE_2\in\Theta_k$, we use Lemma \ref{lemma:bellman_residual},
\begin{equation*}
    \begin{split}
       & V_{1, \pi^*_{\cE_1}}^{\cE_1}(s_1)-V_{1,\pi^*_{\cE_1}}^{\cE_2}(s_1)\\
       =& \sum_{h=1}^H\mathbb E_{\pi^*_{\cE_1}}^{ \cE_2}\left[\mathbb E_{s'\sim P_h^{ \cE_1}(\cdot|s_h^\ell,a_h^\ell)}[V_{h+1,\pi^*_{\cE_1}}^{ \cE_1}(s')]-\mathbb E_{s'\sim P_h^{ \cE_2}(\cdot|s_h^\ell,a_h^\ell)}[V_{h+1,\pi^*_{\cE_1}}^{ \cE_1}(s')]\right]\\
       =&\sum_{h=1}^H\mathbb E_{\pi^*_{\cE_1}}^{ \cE_2}\left[P_h^{ \cE_1}(\cdot|s_h^\ell,a_h^\ell)^{\top}V_{h+1,\pi^*_{\cE_1}}^{ \cE_1}(\cdot)-P_h^{ \cE_2}(\cdot|s_h^\ell,a_h^\ell)^{\top}V_{h+1,\pi^*_{\cE_1}}^{ \cE_2}(\cdot)\right]\\
       &
       +\sum_{h=1}^H\mathbb E_{\pi^*_{\cE_1}}^{ \cE_2}\left[P_h^{ \cE_2}(\cdot|s_h^\ell,a_h^\ell)^{\top}\left(V_{h+1,\pi^*_{\cE_1}}^{ \cE_2}(\cdot)-V_{h+1,\pi^*_{\cE_1}}^{ \cE_1}(\cdot)\right)\right]\\
        \leq&\sum_{h=1}^H\mathbb E_{\pi^*_{\cE_1}}^{ \cE_2}\underbrace{\left[P_h^{ \cE_1}(\cdot|s_h^\ell,a_h^\ell)^{\top}V_{h+1,\pi^*_{\cE_1}}^{ \cE_1}(\cdot)-P_h^{ \cE_2}(\cdot|s_h^\ell,a_h^\ell)^{\top}V_{h+1,\pi^*_{\cE_1}}^{ \cE_2}(\cdot)\right]}_{I_1}\\
       &
       +\sum_{h=1}^H\mathbb E_{\pi^*_{\cE_1}}^{ \cE_2}\underbrace{\left[\left\|V_{h+1,\pi^*_{\cE_1}}^{ \cE_2}(\cdot)-V_{h+1,\pi^*_{\cE_1}}^{ \cE_1}(\cdot)\right\|_2\right]}_{I_2}\,.
    \end{split}
\end{equation*}
The construction follows the following steps.
\begin{itemize}
    \item First, we construct a cover for $I_1$. Since the reward is assumed to be bounded by 1, we have $P_h^{ \cE}(\cdot|s_h^\ell,a_h^\ell)^{\top}V_{h+1,\pi^*_{\cE}}^{ \cE}(\cdot)\in[0, H]$ for any $\cE$. For each $(s, a, h)$, we construct a covering set $\{\cI^1_{sah},\ldots, \cI^m_{sah}\}$ for $[0, 1]$ where $m=H/\varepsilon$. Thus, each set is of length $\varepsilon$.
    \item Second, let $\{\cC_1, \ldots, \cC_{M}\}$ be an $\varepsilon$-covering of a $S$-dimensional $\ell_2$-ball.
    \item Third, we construct the partition $\{\Theta_k\}_{k=1}^K$ in the way that $\cE\in \Theta_k$ if for any $s,a,h$,
    \begin{equation*}
      \left\langle P_h^{ \cE}(\cdot|s,a),V_{h+1,\pi^*_\cE}^{ \cE}(\cdot)\right\rangle\in \cI^{k_1}_{sah}, V_{h+1,\pi^*_\cE}^{ \cE}(\cdot)/H\in \cC_{k_2}\,,
    \end{equation*}
    where $k_1\in[m],k_2\in[M]$. The existence of $k_1,k_2$ holds trivially. 
\end{itemize}
Apparently, $\{\Theta_k\}_{k=1}^K$ is a partition of $\Theta$. For any $k\in[K]$ and $\cE_1, \cE_2\in\Theta_k$, the following holds for any $s, a, h$,
    \begin{equation*}
        \left| \left\langle P_h^{ \cE_1}(\cdot|s,a), V_{h+1,\pi^*_{\cE_1}}^{ \cE_1}(\cdot)\right\rangle-\left\langle P_h^{ \cE_2}(\cdot|s,a), V_{h+1,\pi^*_{\cE_1}}^{ \cE_2}(\cdot)\right\rangle\right|\leq \varepsilon\,,
    \end{equation*}
    and 
    \begin{equation*}
        \left\|\left(V_{h+1,\pi^*_{\cE_1}}^{ \cE_1}(\cdot)-V_{h+1,\pi^*_{\cE_1}}^{ \cE_2}(\cdot)\right)/H\right\|_2\leq \varepsilon\,.
    \end{equation*}
Therefore, we have constructed a partition $\{\Theta_k\}^K_{k=1}$ over $\Theta$ such that for any $\cE_1, \cE_2\in\Theta_k$,
\begin{equation*}
    V_{1, \pi^*_{\cE_1}}^{\cE_1}(s_1)-V_{1,\pi^*_{\cE_1}}^{\cE_2}(s_1)\leq \varepsilon\,,
\end{equation*}
with the covering number bounded by 
 \begin{equation*}
        K\leq \left(\frac{H^2}{\varepsilon}\right)^{SAH}+\left(\frac{H^2}{\varepsilon}+1\right)^S\leq 2(2H^2/\varepsilon)^{SAH}\,.
    \end{equation*}
    This ends the proof.
\end{proof}

\subsection{Proof of Lemma \ref{lemma:cover_linear_MDP}}\label{sec:proof_lemma:cover_linear_MDP}
\begin{proof}
%According to \citet[Proposition 2.3]{jin2020provably}, if $\cE$ is a linear MDP, for any policy $\pi$, there exist weights $\{w_{h,\pi}^\cE\}_{h=1}^H$ such that for any $(s,a,h)$, we have $Q_{h,\pi}^{\cE}(s,a) = \phi(s,a)^{\top}w_{h,\pi}^\cE$.
We construct a partition $\{\Theta_k\}_{k=1}^K$ over $\Theta^{\text{Lin}}$ such that Eq.~\eqref{eqn:cover} holds. For any $\cE_1, \cE_2\in\Theta_k$, by Lemma \ref{lemma:bellman_residual},
\begin{equation*}
    \begin{split}
       & V_{1, \pi^*_{\cE_1}}^{\cE_1}(s_1)-V_{1,\pi^*_{\cE_1}}^{\cE_2}(s_1)\\
       =&\sum_{h=1}^H\mathbb E_{\pi^*_{\cE_1}}^{ \cE_2}\left[P_h^{ \cE_1}(\cdot|s_h^\ell,a_h^\ell)^{\top}V_{h+1,\pi^*_{\cE_1}}^{ \cE_1}(\cdot)-P_h^{ \cE_2}(\cdot|s_h^\ell,a_h^\ell)^{\top}V_{h+1,\pi^*_{\cE_1}}^{ \cE_1}(\cdot)\right]\\
       =&\sum_{h=1}^H\mathbb E_{\pi^*_{\cE_1}}^{ \cE_2}\left[\phi(s_h^\ell, a_h^\ell)^{\top}\sum_{s'} V_{h+1,\pi^*_{\cE_1}}^{ \cE_1}(s') \psi_h^{\cE_1}(s')-\phi(s_h^\ell, a_h^\ell)^{\top}\sum_{s'}V_{h+1,\pi^*_{\cE_1}}^{ \cE_1}(s') \psi_h^{\cE_2}(s')\right]\,,
    \end{split}
\end{equation*}
where the last equation is from the definition of linear MDP. Moreover, since the value function is always bounded by $H$, we have 
\begin{equation*}
    \begin{split}
       & V_{1, \pi^*_{\cE_1}}^{\cE_1}(s_1)-V_{1,\pi^*_{\cE_1}}^{\cE_2}(s_1)
       =H\sum_{h=1}^H\mathbb E_{\pi^*_{\cE_1}}^{ \cE_2}\left[\phi(s_h^\ell, a_h^\ell)^{\top}\left(\sum_{s'} \psi_h^{\cE_1}(s')-\sum_{s'} \psi_h^{\cE_2}(s')\right)\right]\\
       &\leq H\sum_{h=1}^H\mathbb E_{\pi^*_{\cE_1}}^{ \cE_2}\left[\left\|\phi(s_h^\ell, a_h^\ell)\right\|_2\right]\left\|
      \sum_{s'}\psi_h^{\cE_1}(s')-\sum_{s'} \psi_h^{\cE_2}(s')\right\|_2\\
       &\leq H\sum_{h=1}^H\left\|\sum_{s'} \psi_h^{\cE_1}(s')-\sum_{s'} \psi_h^{\cE_2}(s')\right\|_2\,.
    \end{split}
\end{equation*}
From Definition \ref{def:linear_MDP},
\begin{equation*}
   \left\|\sum_{s'}\psi_h^{\cE}(s')\right\|_2\leq C_{\psi}\,.
\end{equation*}
For each $h\in[H]$, let $\{\cC_1^h, \ldots, \cC_{K}^h\}$ be an $\varepsilon$-covering of a $d$-dimensional $\ell_2$-ball. We construct the partition $\{\Theta_k\}_{k=1}^K$ in the way that $\cE\in \Theta_k$ if 
\begin{equation*}
   \frac{1}{C_{\psi}} \sum_{s'} \psi_h^{\cE}(s')\in\cC_k^h\,.
\end{equation*}
Apparently, $\{\Theta_k\}_{k=1}^K$ is a partition of $\Theta$. For any $k\in[K]$ and $\cE_1, \cE_2\in\Theta_k$, the following holds \begin{equation*}
    \left\|\frac{1}{C_{\psi}}\sum_{s'}  \psi_h^{\cE_1}(s')-\frac{1}{C_{\psi}}\sum_{s'} \psi_h^{\cE_2}(s')\right\|_2\leq \varepsilon\,,
\end{equation*}
which implies
\begin{equation*}
    V_{1, \pi^*_{\cE_1}}^{\cE_1}(s_1)-V_{1,\pi^*_{\cE_1}}^{\cE_2}(s_1)\leq H^2C_{\psi}\varepsilon\,.
\end{equation*}
Letting $\varepsilon' = H^2C_{\psi}\varepsilon$, the covering number can be bounded by 
\begin{equation*}
    K\leq \left(\frac{H^2C_{\psi}}{\varepsilon'}+1\right)^{Hd}\,.
\end{equation*}
This ends the proof.
\end{proof}

\subsection{Proof of Lemma \ref{lemma:information_ratio_linearMDP}}\label{sec:proof_lemma:information_ratio_linearMDP}
\begin{proof}
We write $\bar{\cE}_\ell^*$ as the MDP consisted by the mean of posterior measure of $\tilde \cE_\ell^*$. Noting that $\mathbb E_\ell[V_{1,\pi^*_\cE}^{\tilde \cE_\ell}(s_1^\ell)]=\mathbb E_\ell[V_{1,\pi^\ell_{\ts}}^{\tilde \cE_\ell^*}(s_1^\ell)]$,
we decompose the regret as
 \begin{equation*}
     \begin{split}
    \mathbb E_\ell\left[V_{1,\pi^*}^{\tilde \cE_\ell^*}(s_1^\ell)-V_{1,\pi^\ell_{\ts}}^{\tilde \cE_\ell^*}(s_1^\ell)\right]&=
          \mathbb E_\ell\left[V_{1,\pi^*}^{\tilde \cE_\ell^*}(s_1^\ell)-V_{1,\pi^*}^{\tilde \cE_\ell}(s_1^\ell)\right] \\
          &= \underbrace{\mathbb E_\ell\left[V_{1,\pi^*}^{\tilde \cE_\ell^*}(s_1^\ell)-V_{1,\pi^*}^{\bar \cE_\ell^*}(s_1^\ell)\right]}_{I_1} +  \underbrace{\mathbb E_\ell\left[V_{1,\pi^*}^{\bar \cE_\ell^*}(s_1^\ell)-V_{1,\pi^*}^{\tilde \cE_\ell}(s_1^\ell)\right]}_{I_2}\,,
     \end{split}
 \end{equation*} 
According to Lemma \ref{lemma:bellman_residual},
 \begin{equation*}
     \begin{split}
         I_1= \mathbb E_{\ell}\left[\sum_{h=1}^H\mathbb E_{\pi^*}^{\bar \cE_\ell}\left[P_h^{ \tilde \cE_\ell^*}(\cdot|s_h^\ell,a_h^\ell)^{\top}V_{h+1,\pi^*}^{\tilde \cE_\ell^*}(\cdot)-P_h^{\bar \cE_\ell^*}(\cdot|s_h^\ell,a_h^\ell)^{\top}V_{h+1,\pi^*}^{\tilde \cE_\ell^*}(\cdot)\right]\right]\,.
     \end{split}
 \end{equation*}
Using the definition of linear MDPs in Definition \ref{def:linear_MDP},
 \begin{equation*}
 \begin{split}
      I_1&= \mathbb E_{\ell}\left[\sum_{h=1}^H\mathbb E_{\pi^*}^{\bar \cE_\ell^*}\left[\sum_{s'}\phi(s_h^\ell, a_h^\ell)^{\top}\psi_h^{\tilde \cE_\ell^*}(s') V_{h+1,\pi^*}^{\tilde \cE_\ell^*}(s')-\sum_{s'}\phi(s_h^\ell, a_h^\ell)^{\top}\psi_h^{\bar\cE_\ell^*}(s')V_{h+1,\pi^*}^{\tilde \cE_\ell^*}(s')\right]\right]\\
      &= \mathbb E_{\ell}\left[\sum_{h=1}^H\mathbb E_{\pi^*}^{\bar \cE_\ell^*}\left[\phi(s_h^\ell, a_h^\ell)^{\top}\right]\sum_{s'}(\psi_h^{\tilde \cE_\ell^*}(s')- \psi_h^{\bar \cE_\ell^*}(s'))V_{h+1,\pi^*}^{\tilde \cE_\ell^*}(s')\right]\,.
       \end{split}
 \end{equation*}
 Denoting
 \begin{equation*}
     \Sigma_h = \mathbb E_{\ell}\left[\mathbb E_{\pi^*}^{\bar \cE_\ell^*}\left[\phi(s_h^\ell, a_h^\ell)\right]\mathbb E_{\pi^*}^{\bar \cE_\ell^*}\left[\phi(s_h^\ell, a_h^\ell)^{\top}\right]\right]\,,
 \end{equation*}
 we have 
 \begin{equation*}
     I_1=\sum_{h=1}^H\mathbb E_{\ell}\left[\mathbb E_{\pi^*}^{\bar \cE_\ell^*}\left[\phi(s_h^\ell, a_h^\ell)^{\top}\right]\Sigma_h^{-1/2}\Sigma_h^{1/2}\sum_{s'}(\psi_h^{\tilde \cE_\ell^*}(s')- \psi_h^{\bar \cE_\ell^*}(s'))V_{h+1,\pi^*}^{\tilde \cE_\ell^*}(s')\right]
 \end{equation*}
 
 By Cauchy–Schwarz inequality, we have 
 \begin{equation*}
     I_1\leq \sqrt{\sum_{h=1}^H\mathbb E_\ell\left[\left\|\Sigma_h^{1/2}\sum_{s'}(\psi_h^{\tilde \cE_\ell^*}(s')- \psi_h^{\bar \cE_\ell^*}(s'))V_{h+1,\pi^*}^{ \tilde \cE_\ell^*}(s')\right\|_2^2\right]}\sqrt{\sum_{h=1}^H\mathbb E_\ell\left[\left\|\Sigma_h^{-1/2}\mathbb E_{\pi^*}^{\bar \cE_\ell^*}\left[\phi(s_h^\ell, a_h^\ell)\right]\right\|_2^2\right]}\,.
 \end{equation*}
 \begin{itemize}
     \item  For the first part, 
\begin{equation*}
    \begin{split}
        &\sum_{h=1}^H\mathbb E_\ell\left[\left\|\Sigma_h^{1/2}\sum_{s'}(\psi_h^{\tilde \cE_\ell^*}(s')- \psi_h^{\bar \cE_\ell^*}(s'))V_{h+1,\pi^*}^{ \tilde \cE_\ell^*}(s')\right\|_2^2\right]\\
        &=\sum_{h=1}^H\mathbb E_\ell\left[\left(\mathbb E_{\pi^*}^{\bar \cE^*_\ell}\left[\phi(s_h^\ell, a_h^\ell)^{\top}\right]\sum_{s'}(\psi_h^{\tilde \cE_\ell^*}(s')- \psi_h^{\bar \cE_\ell^*}(s'))V_{h+1,\pi^*}^{ \tilde \cE_\ell^*}(s')\right)^2\right]\\
        &=H^2\sum_{h=1}^H\mathbb E_{\ell}\left(\mathbb E_{\pi^*}^{\bar \cE_\ell^*}\left[P_h^{ \tilde \cE_\ell^*}(\cdot|s_h^\ell,a_h^\ell)^{\top}V_{h+1,\pi^*}^{\tilde \cE_\ell^*}(\cdot)/H-P_h^{ \bar \cE_\ell^*}(\cdot|s_h^\ell,a_h^\ell)^{\top}V_{h+1,\pi^*}^{ \tilde \cE_\ell^*}(\cdot)/H\right]\right)^2\,.
    \end{split}
\end{equation*}
Applying Lemma \ref{lemma:information_gain_eqn}, we have
\begin{equation*}
   H^2 \sum_{h=1}^H\mathbb E_\ell\left[\left\|\Sigma_h^{1/2}\sum_{s'}(\psi_h^{\tilde \cE_\ell^*}(s')- \psi_h^{\bar \cE_\ell^*}(s'))V_{h+1,\pi^*}^{ \tilde \cE_\ell^*}(s')\right\|_2^2\right] = H^2\mathbb I_\ell^{\pi^*}\left(\tilde \cE_\ell^*; \cH_{\ell, H}\right)\,.
\end{equation*}
\item For the second part, we can rewrite
 \begin{equation*}
 \begin{split}
    & \mathbb E_\ell\left[\left\|\Sigma_h^{-1/2}\mathbb E_{\pi^*}^{\bar \cE_\ell^*}\left[\phi(s_h^\ell, a_h^\ell)\right]\right\|_2^2\right] \\
    & = \left\langle\mathbb E_\ell\left[\mathbb E_{\pi^*}^{\bar \cE_\ell^*}\left[\phi(s_h^\ell, a_h^\ell)\right]\mathbb E_{\pi^*}^{\bar \cE_\ell^*}\left[\phi(s_h^\ell, a_h^\ell)^{\top}\right]\right], \Sigma_h^{-1}\right\rangle = \langle \Sigma_h, \Sigma_h^{-1}\rangle =d\,.
      \end{split}
 \end{equation*}
 \end{itemize}
Therefore, putting them together,
\begin{equation*}
    I_1\leq \sqrt{H^3d\mathbb I_\ell^{\pi^*}\left(\tilde \cE_\ell^*; \cH_{\ell, H}\right)}\,.
\end{equation*}
The derivation of bounding $I_2$ is similar so we omit the detailed proof here. This ends the proof.
\end{proof}

\subsection{Proof of Theorem \ref{thm:regret_TS}}
Directly using Lemma \ref{lemma:surrogate_learning}, we can decompose 
\begin{equation*}
    \begin{split}
          \BR_L(\pi_{\ts}) &= \sum_{\ell=1}^L\mathbb E\left[\mathbb E_\ell\left[V_{1,\pi^*}^{\cE}(s_1^\ell)-V_{1,\pi^\ell_{\ts}}^\cE(s_1^\ell)\right]\right] \\
          &\leq \sum_{\ell=1}^L\mathbb E\left[\mathbb E_{\ell}\left[V_{1,\pi^*}^{\tilde \cE_\ell^*}(s_1^\ell)-V_{1,\pi^\ell_{\ts}}^{\tilde \cE_\ell^*}(s_1^\ell)\right]\right]+L\varepsilon\\ 
     &\leq \sqrt{\mathbb E\left[\sum_{\ell=1}^L\frac{\left(\mathbb E_\ell\left[V_{1,\pi^*}^{\tilde \cE_\ell^*}(s_1^\ell)-V_{1,\pi^\ell_{\ts}}^{\tilde \cE_\ell^*}(s_1^\ell)\right]\right)^2}{\mathbb I_\ell^{\pi^\ell_{\ts}}(\tilde \cE_\ell^*; \cH_{\ell, H})}\right]}\sqrt{\mathbb E\left[\sum_{\ell=1}^L\mathbb I_\ell^{\pi^\ell_{\ts}}(\tilde \cE_\ell^*; \cH_{\ell, H})\right]}+L\varepsilon\,.
    \end{split}
\end{equation*}
By the construction of $\tilde \cE_\ell^*$,  $\tilde \cE_\ell^*$ and $\cE$ are independent conditional on $\zeta$. Thus $\tilde \cE_\ell^*$ and $\cH_{\ell, H}$ are independent under the law $\mathbb P_{\ell, \pi^\ell_{\ts}}$ given $\zeta$. By the data processing inequality, we have 
\begin{equation*}
    \mathbb I_\ell^{\pi^\ell_{\ts}}(\tilde \cE_\ell^*; \cH_{\ell, H}) \leq \mathbb I_\ell^{\pi^\ell_{\ts}}(\zeta; \cH_{\ell, H})\,.
\end{equation*}
Therefore, 
\begin{equation*}
     \begin{split}
           \BR_L(\pi_{\ts})& \leq \sqrt{\mathbb E\left[\sum_{\ell=1}^L\frac{\left( \mathbb E_{\ell}\left[V_{1,\pi^*}^{\tilde \cE_\ell^*}(s_1^\ell)-V_{1,\pi^\ell_{\ts}}^{\tilde \cE_\ell^*}(s_1^\ell)\right]\right)^2}{\mathbb I_\ell^{\pi^\ell_{\ts}}(\tilde \cE_\ell^*; \cH_{\ell, H})}\right]}\sqrt{\mathbb E\left[\sum_{\ell=1}^L\mathbb I_\ell^{\pi^\ell_{\ts}}(\zeta; \cH_{\ell, H})\right]}+L\varepsilon\\
          &= \sqrt{\mathbb E\left[\sum_{\ell=1}^L\frac{\left( \mathbb E_{\ell}\left[V_{1,\pi^*}^{\tilde \cE_\ell^*}(s_1^\ell)-V_{1,\pi^\ell_{\ts}}^{\tilde \cE_\ell^*}(s_1^\ell)\right]\right)^2}{\mathbb I_\ell^{\pi^\ell_{\ts}}(\tilde \cE_\ell^*; \cH_{\ell, H})}\right]}\sqrt{\mathbb I(\zeta; \cD_{L+1})}+L\varepsilon\,.
     \end{split}
\end{equation*}
This ends the proof.

\section{Proofs of technical lemmas}

\subsection{Proof of Lemma \ref{lemma:information_gain_eqn}}\label{sec:proof_lemma:information_gain_eqn}
\begin{proof}
Using the chain rule of mutual information,
   \begin{equation}\label{def:mutual_information_decom}
    \begin{split}
    \mathbb I_\ell^{\pi}\left(\cE; \cH_{\ell, H}\right) =& \sum_{h=1}^{H}\mathbb E_{\ell}\left[\mathbb I_\ell^{\pi}\left(\cE; (s_{h}^\ell, a_h^\ell, r_{h}^\ell)\big|\cH_{\ell, h-1}\right)\right]\\
    =&\sum_{h=1}^{H}\mathbb E_{\ell}\left[\mathbb I_\ell^{\pi}\left(\cE; s_{h}^\ell\big|\cH_{\ell, h-1}\right)\right] +\sum_{h=1}^{H} \mathbb E_{\ell}\left[ \mathbb I_\ell^{\pi}\left(\cE; a_{h}^\ell\big|s_{h}^\ell, \cH_{\ell, h-1}\right)\right]\\
    &+\sum_{h=1}^{H} \mathbb E_{\ell}\left[ \mathbb I_\ell^{\pi}\left(\cE; r_{h}^\ell\big|s_{h}^\ell, a_h^\ell,  \cH_{\ell, h-1}\right)\right]\,.
    \end{split}
    \end{equation}
    
\begin{itemize}
    \item For the first term in Eq.~\eqref{def:mutual_information_decom}, by the definition of conditional mutual information and Markov property, we have  
     \begin{equation}\label{eqn:bound_mutual_information_s}
    \begin{split}
  &\mathbb I_\ell^{\pi}\left(\cE; s_{h}^\ell\big|\cH_{\ell, h-1}\right) \\
  &= \int D_{\KL}\left(\mathbb P_{\ell, \pi}\left(s_h^\ell=\cdot|\cH_{\ell, h-1}, \cE\right)||\mathbb P_{\ell, \pi}\left(s_h^\ell = \cdot|\cH_{\ell, h-1}\right)\right)\d \mathbb P_\ell(\cE|\cH_{\ell, h-1})\\
    & = \int D_{\KL}\left(P_h^{\cE}\left(\cdot|s_{h-1}^\ell, a_{h-1}^\ell\right)||\mathbb P_{\ell, \pi}\left(s_h^\ell = \cdot|\cH_{\ell, h-1}\right)\right)\d \mathbb P_\ell(\cE|\cH_{\ell, h-1})\,.
        \end{split}
\end{equation}
Since the priors of transition probability kernel are independent over different layers, $\mathbb P_\ell(\cE|\cH_{\ell, h-1})=\mathbb P_\ell(\cE)$ such that
\begin{equation}\label{eqn:s_h}
\begin{split}
    \mathbb P_{\ell, \pi}\left(s_h^\ell = \cdot|\cH_{\ell, h-1}\right) &=\int \mathbb P_{\ell, \pi}\left(s_h^\ell=\cdot|\cH_{\ell, h-1}, \cE\right) \d \mathbb P_\ell(\cE|\cH_{\ell, h-1})\\
    &=\int P_h^\cE(\cdot|s_{h-1}^\ell, a_{h-1}^\ell) \d \mathbb P_\ell(\cE|\cH_{\ell, h-1})\\
    &=\int P_h^\cE(\cdot|s_{h-1}^\ell, a_{h-1}^\ell) \d \mathbb P_\ell(\cE)\\
     &= P_h^{\bar \cE_\ell}\left(\cdot|s_{h-1}^\ell,a_{h-1}^\ell\right)\,,
\end{split}
\end{equation}
where the last equation is by the  definition of probability kernel $P_h^{\bar \cE_\ell}$.
Combining Eqs.~\eqref{eqn:bound_mutual_information_s} and \eqref{eqn:s_h} together,
\begin{equation*}
 \mathbb I_\ell^{\pi}\left(\cE; s_{h}^\ell\big|\cH_{\ell, h-1}\right) =\int D_{\KL}\left(P_h^{\cE}(\cdot|s_{h-1}^\ell,a_{h-1}^\ell)||P_h^{\bar \cE_\ell}(\cdot|s_{h-1}^\ell,a_{h-1}^\ell)\right)\d \mathbb P_\ell(\cE)\,.
\end{equation*}
Therefore,
\begin{equation*}
\begin{split}
     & \mathbb E_{\ell}\left[\mathbb I_\ell^{\pi}\left(\cE; s_{h}^\ell\big|\cH_{\ell, h-1}\right)\right]\\
     &=\mathbb E_{\ell}\left[\int D_{\KL}\left(P_h^{\cE}(\cdot|s_{h-1}^\ell,a_{h-1}^\ell)||P_h^{\bar \cE_\ell}(\cdot|s_{h-1}^\ell,a_{h-1}^\ell)\right)\d \mathbb P_\ell(\cE)\right]\\
       &=\sum_{(s,a)}\mathbb P_{\ell, \pi}(s_{h-1}^\ell=s, a_{h-1}^\ell=a)\int D_{\KL}\left(P_h^{\cE}(\cdot|s,a)||P_h^{\bar \cE_\ell}(\cdot|s,a)\right)\d \mathbb P_\ell(\cE)\\
        &=\sum_{(s,a)}\int\mathbb P_{\ell, \pi}(s_{h-1}^\ell=s, a_{h-1}^\ell=a|\cE)\d \mathbb P_\ell(\cE)\int D_{\KL}\left(P_h^{\cE}(\cdot|s,a)||P_h^{\bar \cE_\ell}(\cdot|s,a)\right)\d \mathbb P_\ell(\cE)\,.
  \end{split}
\end{equation*}   
Using the linearity of expectation and the independence of priors over different layers, we can show
\begin{equation*}
    \int\mathbb P_{\ell, \pi}(s_{h-1}^\ell=s, a_{h-1}^\ell=a|\cE)\d \mathbb P_\ell(\cE) = \mathbb P^{\bar\cE_\ell}_{\pi}(s_{h-1}^\ell=s, a_{h-1}^\ell=a)\,.
\end{equation*}
\begin{comment}
{\botao
To see why the above holds,
\begin{equation*}
    \begin{split}
      &   \int\mathbb P_{\ell, \pi}(s_{h-1}^\ell=s, a_{h-1}^\ell=a|\cE)\d \mathbb P_\ell(\cE)\\
      &=\int\mathbb P_{\ell, \pi}( a_{h-1}^\ell=a|s_{h-1}^\ell=s,\cE)\sum_{s',a'}\mathbb P_{\ell, \pi}(s_{h-1}^\ell=s|s_{h-2}^\ell=s', a_{h-2}^\ell=a',\cE)\mathbb P_{\ell,\pi}(s_{h-2}^\ell=s', a_{h-2}^\ell=a'|\cE)\d \mathbb P_\ell(\cE)\\
      &=\int\pi_{h-1}(a|s)\sum_{s',a'}P_{h-2}^{\cE}(s|s',a')\mathbb P_{\ell,\pi}(s_{h-2}^\ell=s', a_{h-2}^\ell=a'|\cE)\d \mathbb P_\ell(\cE)
    \end{split}\,.
\end{equation*}
On the other hand,
\begin{equation*}
\begin{split}
    \mathbb P^{\bar\cE_\ell}_{\pi}(s_{h-1}^\ell=s, a_{h-1}^\ell=a) &= \pi_{h-1}(a|s)\sum_{s',a'}\mathbb P^{\bar\cE_\ell}_{\pi}(s_{h-1}^\ell = s|s_{h-2}^\ell=s', a_{h-2}^\ell=a')\mathbb P^{\bar\cE_\ell}_{\pi}(s|s_{h-2}^\ell=s', a_{h-2}^\ell=a')\\
    &=\pi_{h-1}(a|s)\sum_{s',a'}\int P_{h-2}^{\cE}(s|s',a')\d \mathbb P_\ell(\cE)\mathbb P^{\bar\cE_\ell}_{\pi}(s|s_{h-2}^\ell=s', a_{h-2}^\ell=a')
    \end{split}
\end{equation*}
}
\end{comment}

This implies
  \begin{equation*}
\begin{split}      
 & \mathbb E_{\ell}\left[\mathbb I_\ell^{\pi}\left(\cE; s_{h}^\ell\big|\cH_{\ell, h-1}\right)\right]   \\ &=\sum_{(s,a)}\mathbb P^{\bar\cE_\ell}_{\pi}(s_{h-1}^\ell=s, a_{h-1}^\ell=a)\int D_{\KL}\left(P_{h}^{\cE}(\cdot|s,a)||P_h^{\bar \cE}(\cdot|s,a)\right)\d \mathbb P_\ell(\cE)\\
      &=\int\mathbb E_{\pi}^{\bar \cE_\ell}\left[ D_{\KL}\left(P_h^{\cE}(\cdot|s_{h-1}^\ell,a_{h-1}^\ell)||P_h^{\bar \cE_\ell}(\cdot|s_{h-1}^\ell,a_{h-1}^\ell)\right)\right]\d \mathbb P_\ell(\cE)\\
      &= \mathbb E_{\ell}\left[\mathbb E_{\pi}^{\bar \cE_\ell}\left[D_{\KL}\left(P_h^{\cE}(\cdot|s_{h-1}^\ell,a_{h-1}^\ell)||P_h^{\bar \cE_\ell}(\cdot|s_{h-1}^\ell,a_{h-1}^\ell)\right)\right]\right]\,,
\end{split}
\end{equation*}
where $\mathbb E_{\pi}^{\bar \cE_\ell}$ is taken with respect to $s_{h-1}^\ell,a_{h-1}^\ell$ and $\mathbb E_\ell$ is taken with respect to $\cE$.
   \item For the second term of Eq.~\eqref{def:mutual_information_decom}, we have 
   \begin{equation*}
   \begin{split}
       &\mathbb I_\ell^{\pi}\left(\cE; a_{h}^\ell\big|s_{h}^\ell, \cH_{\ell, h-1}\right)\\
       &= \int D_{\KL}\left(\mathbb P_{\ell, \pi}\left(a_h^\ell=\cdot|\cH_{\ell, h-1},s_h^\ell, \cE\right)||\mathbb P_{\ell, \pi}\left(a_h^\ell=\cdot|s_h^\ell, \cH_{\ell, h-1}\right)\right)\d \mathbb P_\ell(\cE)\,.
       \end{split}
   \end{equation*}
  When $\cH_{\ell, h-1}^{\pi}$ is fixed, both sides of the KL term are equal to $\pi(\cdot|s_h^\ell)$ and thus $\mathbb I_\ell^{\pi}(\cE; a_{h}^\ell|s_{h}^\ell, \cH_{\ell, h-1})=0$.
  \item For the third term of Eq.~\eqref{def:mutual_information_decom}, since the reward function is deterministic and known, we have 
  \begin{equation*}
      \begin{split}
          \mathbb I_\ell^{\pi}\left(\cE; r_{h}^\ell\big|s_{h}^\ell, a_h^\ell,  \cH_{\ell, h-1}^{\pi}\right)= 0\,.
      \end{split}
  \end{equation*}
\end{itemize}    
This suffices to show
\begin{equation}\label{eqn:I1}
\begin{split}
 \sum_{h=1}^H \mathbb E_{\ell}\mathbb E_{\pi^\ell_{\ts}}^{\bar \cE_\ell}\left[D_{\KL}\left(P_h^{\cE}(\cdot|s_h,a_h)||P_h^{\bar \cE_\ell}(\cdot|s_h,a_h)\right)\right]=\mathbb I_\ell^{\pi}\left(\cE; \cH_{\ell, H}\right)\,.
    \end{split}
\end{equation}
This ends the proof.

\end{proof}

\section{Supporting lemmas}

\begin{lemma}[Lemma 1 in \cite{dong2018information}]\label{lemma1}
Let $\{a_i\}_{i=1}^N$ and $\{b_i\}_{i=1}^N$ be two sequences of real numbers, where $N<\infty$. Let $\{p_i\}_{i=1}^N$ be such that $p_i\geq 0$ for all $i$ and $\sum_{i=1}^N p_i=1$. Then there exists indices $j,k\in[N]$ and $r\in[0, 1]$ such that
\begin{equation*}
    ra_j+(1-r)a_k\leq \sum_{i=1}^Na_i p_i, rb_j+(1-r)b_k\leq \sum_{i=1}^Lb_i p_i\,.
\end{equation*}
\end{lemma}

\begin{lemma}[Fact 9 in \cite{russo2014learning}]\label{lemma:pinsker}
For any distribution $P$ and $Q$ such that $P$ is absolutely continuous with respect to $Q$, any random variable $X:\Omega \to \cX$ and any $g:\cX\to \mathbb R$ such that $\sup g-\inf g\leq 1$, we have 
\begin{equation*}
    \mathbb E_P[g(x)]-\mathbb E_Q[g(x)]\leq \sqrt{\frac{1}{2}D_{\KL}(P||Q)}\,,
\end{equation*}
which is a variant of Pinsker's inequality.
\end{lemma}

\begin{lemma}\label{lemma:bellman_residual}
For any two environments $\cE, \cE'$, any policy $\pi$ and a fixed set of reward functions $\{r_h\}_{h=1}^H$, we have
\begin{equation*}
\begin{split}
     &V_{1, \pi}^{\cE}(s_1)-V_{1, \pi}^{\cE'}(s_1)=\sum_{h=1}^H\mathbb E_{\pi}^{ \cE'}\left[\mathbb E_{s'\sim P_h^{ \cE}(\cdot|s_h,a_h)}[V_{h+1,\pi}^{ \cE}(s')]-\mathbb E_{s'\sim P_h^{ \cE'}(\cdot|s_h,a_h)}[V_{h+1,\pi}^{ \cE}(s')]\right],
\end{split}
\end{equation*}
where we define $V_{H+1,\pi^*}^{ \cE}(\cdot)=0$ and the outer expectation $\mathbb E_{\pi}^{\cE'}$ is with respect to $s_h,a_h$.
\end{lemma}
\begin{proof}
Similar proofs can be found in \cite{osband2013more, foster2021statistical}. For the self-completeness, we include a full proof here.
First, we realize
\begin{equation}\label{eqn:bellman_dec_1}
    \begin{split}
      &  \sum_{h=1}^H\mathbb E_{\pi}^{ \cE'}\left[Q_{h,\pi}^{ \cE}(s_h,a_h)-r_h(s_h, a_h)-V_{h+1,\pi}^{ \cE}(s_{h+1})\right]\\
      &=\sum_{h=1}^H\mathbb E_{\pi}^{ \cE'}\left[Q_{h,\pi}^{ \cE}(s_h,a_h)-V_{h+1,\pi}^{ \cE}(s_{h+1})\right]-\sum_{h=1}^H\mathbb E_{\pi}^{ \cE'}\left[r_h(s_h, a_h)\right]\\
      &=\sum_{h=1}^H\mathbb E_{\pi}^{ \cE'}\left[Q_{h,\pi}^{ \cE}(s_h,a_h)-V_{h+1,\pi}^{ \cE}(s_{h+1})\right]-V_{1,\pi}^{\cE'}(s_1)\,.
    \end{split}
\end{equation}
Since $V_{h,\pi}^\cE(s)=\mathbb E_{a\sim \pi_h(\cdot|s)}[Q_{h,\pi}^\cE(s,a)]$, we have 
\begin{equation}\label{eqn:bellman_dec_2}
\begin{split}
     &\sum_{h=1}^H\mathbb E_{\pi}^{ \cE'}\left[Q_{h,\pi}^{ \cE}(s_h,a_h)-V_{h+1,\pi}^{ \cE}(s_{h+1})\right]= \sum_{h=1}^H\mathbb E_{\pi}^{ \cE'}\left[V_{h,\pi}^{ \cE}(s_{h})-V_{h+1,\pi}^{ \cE}(s_{h+1})\right]\\
     &= \mathbb E_{\pi}^{ \cE'}\left[V_{1,\pi}^{ \cE}(s_{1})\right]= V_{1,\pi}^{ \cE}(s_{1})\,.
\end{split}
\end{equation}
Using the Bellman equation, we have 
\begin{equation*}
\begin{split}
     & V_{1, \pi}^{\cE}(s_1)-V_{1, \pi}^{\cE'}(s_1)\\
     &=\sum_{h=1}^H\mathbb E_{\pi}^{ \cE'}\left[\mathbb E_{s'\sim P_h^{ \cE}(\cdot|s_h,a_h)}[V_{h+1,\pi}^{ \cE}(s')]-\mathbb E_{s'\sim P_h^{ \cE'}(\cdot|s_h,a_h)}[V_{h+1,\pi}^{ \cE}(s')]\right]\,.
\end{split}
\end{equation*}
This ends the proof.
\end{proof}

\end{document}